\documentclass{sig-alternate-2013}
\usepackage{amsmath,amssymb}
\usepackage[ruled,linesnumbered]{algorithm2e}
\usepackage{algorithmic}
\usepackage{paralist} 
\usepackage{color}
\usepackage{multirow}
\usepackage{rotating}
\usepackage{times}
\usepackage{url}
\usepackage[mathscr]{eucal} 
\usepackage{amsbsy} 
\usepackage{subfigure}
\usepackage{booktabs}
\usepackage[table]{xcolor}
\usepackage{tikz}
\usetikzlibrary{snakes}
\usepackage{multirow}
\usepackage{epstopdf}
\usepackage{balance}
\usepackage{tablefootnote}
\usepackage{empheq} 
\usepackage{moresize}

\usepackage{graphicx} 
\usepackage{bmpsize}

\usepackage{enumitem}
\setlist{nolistsep}

\usepackage[font=small,labelfont=bf,skip=0pt]{caption}

\DeclareCaptionType{copyrightbox}
\usepackage{url}

\newcounter{ALC@tempcntr}
\newcommand{\subfigurewidth}{0.18\textwidth}

\newtheorem{infproblem}{Informal~Problem}

\newtheorem{claim}{Claim}

\newcommand{\hide}[1]{}

\newcommand{\bit}{\begin{compactitem}}
\newcommand{\eit}{\end{compactitem}}
\newcommand{\ben}{\begin{compactenum}}
\newcommand{\een}{\end{compactenum}}

\newcommand{\method}{\textsc{AutoTen}\xspace}
\newcommand{\methodplain}{{AutoTen}\xspace}
\newcommand{\kronmatvec}{\textsc{KronMatVec}\xspace}
\newcommand{\cpapr}{\texttt{CP\_APR}\xspace}
\newcommand{\cpnmu}{\texttt{CP\_NMU}\xspace}

\newcommand{\enron}{\texttt{ENRON}\xspace}
\newcommand{\reality}{\texttt{Reality Mining}\xspace}
\newcommand{\facebook}{\texttt{Facebook}\xspace}
\newcommand{\taxi}{\texttt{Taxi}\xspace}
\newcommand{\dblp}{\texttt{DBLP}\xspace}
\newcommand{\netflix}{\texttt{Netflix}\xspace}
\newcommand{\amazonco}{\texttt{Amazon co-purchase}\xspace}
\newcommand{\amazonmeta}{\texttt{Amazon metadata}\xspace}
\newcommand{\yelp}{\texttt{Yelp}\xspace}
\newcommand{\airport}{\texttt{Airport}\xspace}

\newcommand{\tensor}[1]{\underline{\mathbf{#1}}}

\newcommand{\codeurl}{\url{http://www.cs.cmu.edu/~epapalex/src/AutoTen.zip}}

\def\x{{\mathbf x}}

\begin{document}
\permission{}
\conferenceinfo{}{}
\copyrightetc{}
\clubpenalty=10000
\widowpenalty = 10000

\title{Automatic Unsupervised Tensor Mining with Quality Assessment}

\numberofauthors{1}

\author{
\alignauthor
Evangelos E. Papalexakis\\
       \affaddr{Carnegie Mellon University}\\
       \email{epapalex@cs.cmu.edu}
}

\maketitle

%
\maketitle
\begin{abstract}
A popular tool for unsupervised modelling and mining multi-aspect data is tensor decomposition. In an exploratory setting, where and no labels or ground truth are available how can we automatically decide how many components to extract? How can we assess the quality of our results, so that a domain expert can factor this quality measure in the interpretation of our results? 
In this paper, we introduce \method, a novel automatic unsupervised tensor mining algorithm with minimal user intervention, which leverages and improves upon heuristics that assess the result quality. We extensively evaluate \method's performance on synthetic data, outperforming existing baselines on this very hard problem. Finally, we apply \method on a variety of real datasets, providing insights and discoveries. We view this work as a step towards a fully automated, unsupervised tensor mining tool that can be easily adopted by practitioners in academia and industry.
\end{abstract}


\section{Introduction}
\label{sec:introduction}

Tensor decompositions and their applications in mining multi-aspect datasets are ubiquitous and ever increasing in popularity. Data Mining application of these techniques has been  largely pioneered by the work of Kolda et al. \cite{kolda2005higher} where the authors introduce a topical aspect to a graph between webpages, and extend the popular HITS algorithm in that scenario. Henceforth, the field of multi-aspect/tensor data mining has witnessed rich growth with prime examples of applications being citation networks \cite{kolda2008scalable}, computer networks\cite{kolda2008scalable,maruhashi2011multiaspectforensics,papalexakis2012parcube}, Knowledge Base data \cite{kang2012gigatensor,papalexakis2012parcube,chang2013multi,chang2014typed}, and social networks \cite{bader2007temporal,kolda2008scalable,lin2009metafac,papalexakis2012parcube,jiang2014fema}, to name a few.

Tensor decompositions are undoubtedly a very powerful analytical tool with a rich variety of applications. However there exist research challenges in the field of data mining that need to be addressed, in order for tensor decompositions to claim their position as a de-facto tool for practicioners.

One challenge, which has received considerable attention, is the one of making tensor decompositions scalable to today's web scale. For instance, Facebook has around 2 billion users at the time of writing of this paper and is ever growing, and making tensor decompositions able to work on even small portions of the entire Facebook network is imperative for the adoption of these techniques by such big players. Very frequently, data that fall under the aforementioned category turn out to be highly sparse; the reason is that, e.g. each person on Facebook interacts with only a few hundreds of the users. Computing tensor decompositions for highly sparse scenarios is a game changer, and exploiting sparsity is key in scalability. The work of Kolda et al. \cite{kolda2005higher,SIAM-67648} introduced the first such approach of exploiting sparsity for scalability. Later on, distributed approaches based on the latter formulation \cite{kang2012gigatensor}, or other scalable approaches \cite{papalexakis2012parcube,erdos2013walk,beutelflexifact,de2014distributed} have emerged. By no means do we claim that scalability is a solved problem, however, we point out that there has been significant attention to it.

\begin{figure}[!ht]
	\begin{center}
		\includegraphics[width = 0.45\textwidth]{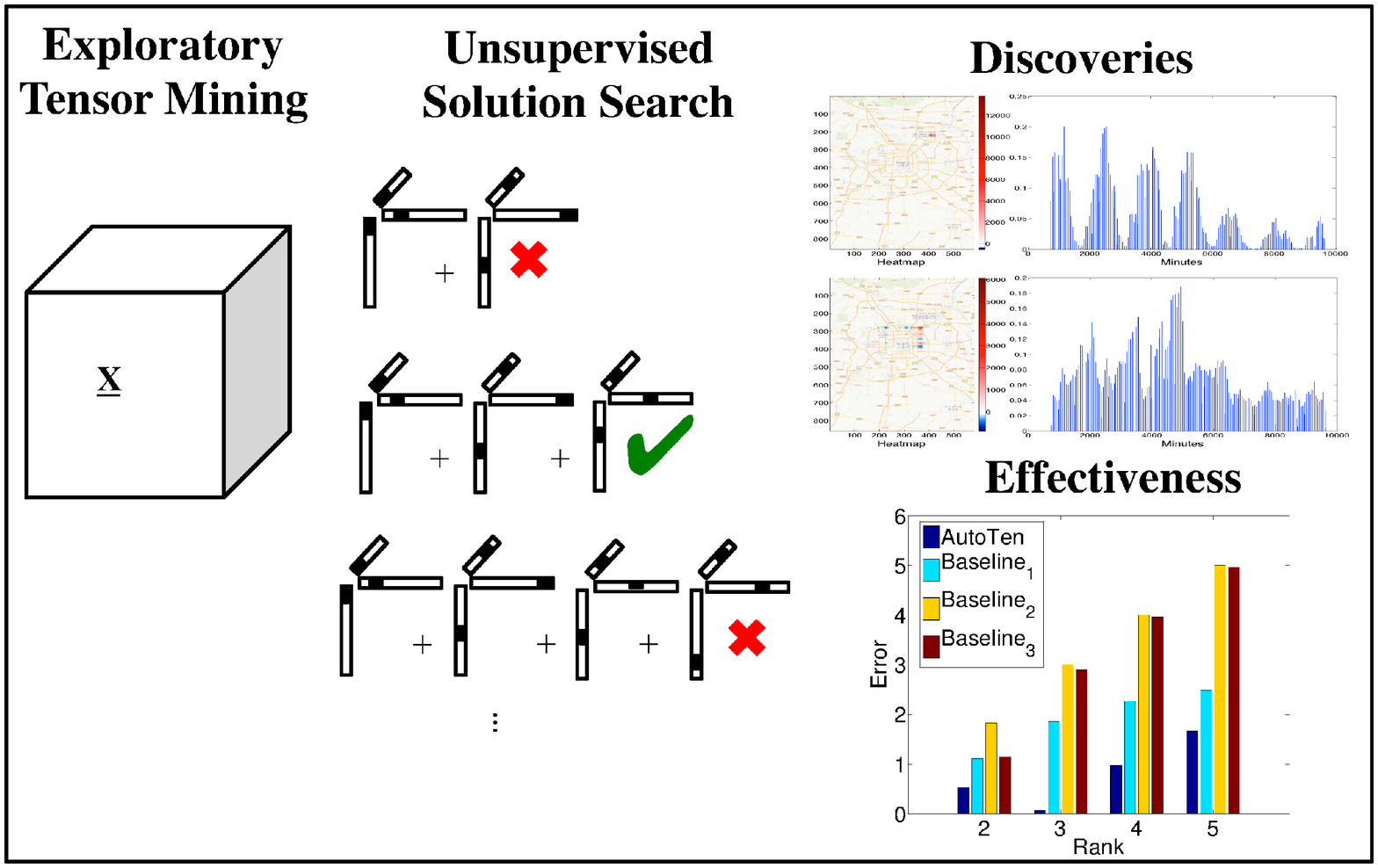}
	\end{center}
	\caption{Starting from an unsupervised, exploratory application, \method automatically determines a solution with high quality, outperforming existing baselines, and enables discoveries in real data.}
\end{figure}

The main focus of this work, however, is on another, relatively less explored territory; that of assessing the {\em quality} of a tensor decomposition. In a great portion of tensor data mining, the task is exploratory and unsupervised: we are given a dataset, usually without any sort of ground truth, and we seek to extract {\em interesting} patterns or concepts from the data. It is crucial, therefore, to know whether a pattern that we extract actually models the data at hand, or whether it is merely modelling noise in the data. Especially in the age of Big Data, where feature spaces can be vast, it is imperative to have a measure of quality and avoid interpreting noisy, random variation that always exists in the data. Determining the ``right'' number of components in a tensor is a very hard problem \cite{hillar2013most}. This is why, many seminal exploratory tensor mining papers, understandably, set the number of components manually \cite{kolda2006tophits,sun2006beyond,bader2007temporal,kolda2008scalable}. When there is a specific task at hand, e.g. link prediction \cite{dunlavy2011temporal},  recommendation \cite{rendle2010pairwise}, and supervised learning \cite{tao2005supervised,he2014dusk}, that entails some measure of success, then there is some  procedure (e.g. cross-validation) for selecting a good number of latent components which unfortunately cannot generalize to the case where labels or ground truth are absent.

However, not all hope is lost. There have been very recent approaches following the Minimum Description Length (MDL) principle \cite{araujo2014com2,metzler2015clustering}, where the MDL cost function usually depends heavily on the application at hand (e.g. community detection or boolean tensor clustering respectively). Additionally, there have been Bayesian approaches \cite{zhaobayesian} that, as in the MDL case, do not require the number of components as input. These approaches are extremely interesting, and we reserve their deeper investigation in future work, however in this work, we choose to operate on top of a different, very intuitive approach which takes into account properties of the PARAFAC decomposition \cite{PARAFAC} and is {\em application independent}, requiring no prior knowledge about the data; there exists highly influential work in the Chemometrics literature \cite{bro1998multi} that introduces heuristics for determining a good rank for tensor decompositions. Inspired by and drawing from \cite{bro1998multi}, we provide a comprehensive method for mining multi-aspect datasets using tensor decompositions.

Our contributions are:
\begin{itemize}\setlength{\itemsep}{0.0mm}
	\item{\bf Algorithms} We propose \method, a comprehensive methodology on mining multi-aspect datasets using tensors, which minimizes manual trial-and-error intervention and provides quality characterization of the solution (Section \ref{sec:method_final}). Furthermore, we extend the quality assessment heuristic of \cite{bro1998multi} assuming KL-divergence, which has been shown to be more effective in highly sparse, count data \cite{chi2012tensors} (Section \ref{sec:method_corcondia}).
	\item{\bf Evaluation \& Discovery} We conduct a large scale study on 10 real datasets, exploring the structure of hidden patterns within these datasets. To the best of our knowledge, this is the first such broad study. (Section \ref{sec:study1}). As a data mining case study, we apply \method to two real datasets discovering meaningful patterns (Section \ref{sec:study2}). Finally, we extensively evaluate our proposed method in synthetic data (Section \ref{sec:experiments}).
\end{itemize}
In order to encourage reproducibility, most of the datasets used are public, and we make our code publicly available at {\small \\ \codeurl}. 

\section{Background}
\label{sec:formulation}
Table \ref{tab:symbols} provides an overview of the notation used in this and subsequent sections.

\begin{table}[htbp]
\small
\begin{center}
\begin{tabular}{l | l}
    {\bf Symbol}      & {\bf Definition} \\ \hline
	$\tensor{X}, \mathbf{X}, \mathbf{x}, x$ & Tensor, matrix, column vector, scalar \\ \hline
	$\circ$ & outer product \\ \hline
	$vec(~)$ & vectorization operator \\\hline
	$\otimes$ & Kronecker product \\ \hline
	$\mathbf{X}^\dag$ & Moore-Penrose pseudoinverse \\\hline
	$\ast ~~\oslash$ & element-wise multiplication and division \\\hline
    $\mathbf{A}^\dag$ & Moore-Penrose Pseudoinverse of $\mathbf{A}$  \\\hline
    $D_{KL}(a||b)$ & KL-Divergence \\\hline
    $\|\mathbf{A}\|_F$ & Frobenius norm \\\hline
    \multirow{2}{*}{\kronmatvec} & efficient computation of  \\ & $\mathbf{y} = \left(\mathbf{A}_1 \otimes \mathbf{A}_2 \otimes \cdots \otimes \mathbf{A}_n \right)\mathbf{x} $
    \cite{buis1996efficient} \\\hline
    $\mathbf{x}(i)$ & $i$-th entry of $\mathbf{x}$ (same for matrices and tensors)\\ \hline
    $\mathbf{X}(:,i)$ & spans the entire $i$-th column of $\mathbf{X}$ (same for tensors) \\ \hline
    $\mathbf{x}^{(k)}$ & value at the $k$-th iteration \\ \hline 
    \cpnmu & non-negative, Frobenius norm PARAFAC \cite{tensortoolbox}\\ \hline
    \cpapr & KL-Divergence PARAFAC \cite{chi2012tensors} \\ \hline
\end{tabular}
\caption{Table of symbols}
\label{tab:symbols}
\end{center}
\end{table}

\subsection{Brief Introduction to Tensor Decompositions}

Given a tensor $\tensor{X}$, we can decompose it according to the CP/PARAFAC decomposition \cite{PARAFAC} (henceforth referred to as PARAFAC) as a sum of rank-one tensors:
\[
	\tensor{X} \approx \displaystyle{ \sum_{f=1}^{F} \mathbf{a}_f \circ \mathbf{b}_f \circ \mathbf{c}_f }
\]
where the $(i,j,k)$ entry of $\mathbf{a}_r\circ \mathbf{b}_r \circ \mathbf{c}_r$ is $\mathbf{a}_r(i)\mathbf{b}_r(j)\mathbf{c}_r(k)$. Usually, PARAFAC is represented in its matrix form $[\mathbf{A,B,C} ]$, where the columns of matrix $\mathbf{A}$ are the $\mathbf{a}_r$ vectors (and accordingly for $\mathbf{B,C}$). 
The PARAFAC decomposition is especially useful when we are interested in extracting the true latent factors that generate the tensor. In this work, we choose the PARAFAC decomposition as our tool, since it admits a very intuitive interpretation of its latent factors; each component can be seen as soft co-clustering of the tensor, using the high values of vectors $\mathbf{a}_r, \mathbf{b}_r, \mathbf{c}_r$ as the membership values to co-clusters.

Another very popular Tensor decomposition is Tucker3 \cite{kroonenberg1980principal}, where a tensor is decomposed into rank-one factors times a core tensor:
	\[
	\tensor{X} \approx \displaystyle{ \sum_{p=1}^{P} \sum_{q=1}^{Q} \sum_{r=1}^{R} \tensor{G}(p,q,r)\mathbf{u}_p \circ \mathbf{v}_q \circ\mathbf{w}_r }
	\]
where $\mathbf{U,V,W}$ are orthogonal.
The Tucker3 model is especially used for compression. Furthermore, PARAFAC can be seen as a restricted Tucker3 model, where the core tensor $\tensor{G}$ is super-diagonal, i.e. non-zero values are only in the entries where $i=j=k$. This observation will be useful in order to motivate the CORCONDIA diagnostic.

Finally, there also exist more expressive, but harder to interpret models, such as the Block Term Decomposition (BTD) \cite{de2008decompositions} , however, we reserve future work for their investigation.

\subsection{Brief Introduction to CORCONDIA}
As outlined in the Introduction, in the chemometrics literature, there exists a very intuitive heuristic by the name of CORCONDIA \cite{bro2003new}, which can serve as a guide in judging how well a given PARAFAC decomposition is modelling a tensor.

In a nutshell, the idea behind CORCONDIA is the following: Given a tensor $\tensor{X}$ and its PARAFAC decomposition $\mathbf{A,B,C}$, one could imagine fitting a Tucker3 model where matrices $\mathbf{A,B,C}$ are the factors of the Tucker3 decomposition and $\tensor{G}$ is the core tensor (which we need to solve for). Since, as we already mentioned, PARAFAC can be seen as a restricted Tucker3 decomposition with super-diagonal core tensor, if our PARAFAC modelling of $\tensor{X}$ using $\mathbf{A,B,C}$ is modelling the data well, the core tensor $\tensor{G}$ should be as close to super-diagonal as possible. If there are deviations from the super-diagonal, then this is a good indication that our PARAFAC model is somehow flawed (either the decomposition rank is not appropriate, or the data do not have the appropriate structure).
We can pose the problem as the following least squares problem:
\[
\min_{\tensor{G}} 	\| vec\left( \tensor{X} \right) - \left( \mathbf{A}\otimes \mathbf{B} \otimes \mathbf{C} \right) vec \left( \tensor{G} \right)\|_F^2
\]
with the least squares solution:
$
	vec\left(\tensor{G} \right) = \left( \mathbf{A}\otimes \mathbf{B} \otimes \mathbf{C} \right)^\dag vec \left( \tensor{X} \right)
$

After computing $\tensor{G}$, the CORCONDIA diagnostic can be computed as
{\small
$
	\displaystyle{ c = 100\left(1 - \frac{\sum_{i=1}^{F} \sum_{j=1}^{F} \sum_{k=1}^{F} \left(\tensor{G}(i,j,k) - \tensor{I}(i,j,k)\right)^2 }{F}  \right)},
$
}
where $\tensor{I}$ is a super-diagonal tensor with ones on the $(i,i,i)$ entries. For a perfectly super-diagonal $\tensor{G}$ (i.e. perfect modelling), $c$ will be 100. 
One can see that for rank-one models, the metric will always be 100, because the rank one component can trivially produce a single element ``super-diagonal'' core; thus, CORCONDIA is applicable for rank two or higher.
According to \cite{bro2003new}, values below 50 show some imperfection in the modelling or the rank selection; the value can also be negative, showing severe problems with the modelling. In \cite{bro2003new}, some of the chemical data analyzed have perfect, low rank PARAFAC structure, and thus expecting $c>50$ is reasonable. In many data mining applications, however, due to high data sparsity, data cannot have such perfect structure, but an {\em approximation} thereof using a low rank model is still very valuable. Thus, in our case, we expand the spectrum of acceptable solutions with reasonable quality to include smaller, positive values of $c$ (e.g. 20 or higher).

\subsection{Scaling Up CORCONDIA}
As  we mention in the Introduction CORCONDIA as it's introduced in \cite{bro1998multi} is suitable for small and dense data. However, this contradicts the area of interest of the vast majority of data mining applications. To that end, very recently \cite{papalexakis2015fastcorcondia} we extended CORCONDIA to the case where our data are large but sparse, deriving a fast and efficient algorithm. Key behind \cite{papalexakis2015fastcorcondia} is avoiding to pseudoinvert $\left( \mathbf{A \otimes B \otimes C} \right)$

In order to achieve the above, we need to reformulate the computation of CORCONDIA.
The pseudoinverse 
$
\left( \mathbf{A} \otimes \mathbf{B} \otimes \mathbf{C} \right)^\dag
$
can be rewritten as
{\small
\[
\left( \mathbf{V_a} \otimes \mathbf{V_b} \otimes \mathbf{V_c} \right) \left( \mathbf{\Sigma_a}^{-1} \otimes \mathbf{\Sigma_b}^{-1} \otimes \mathbf{\Sigma_c}^{-1} \right)  \left( \mathbf{U_a}^T \otimes \mathbf{U_b}^T \otimes \mathbf{U_c}^T \right) 
\]
}
where $\mathbf{A} = \mathbf{U_a \Sigma_a {V_a}^T}$, $\mathbf{B} = \mathbf{U_b \Sigma_b {V_b}^T}$, and $\mathbf{C} = \mathbf{U_c \Sigma_c {V_c}^T}$ (i.e. the respective Singular Value Decompositions).

After we rewrite the least squares problem in the above form, we can efficiently carry out a series of Kronecker products times a vector very efficiently, {\em without} materializing the (potentially big) Kronecker product. In \cite{papalexakis2015fastcorcondia} we use the algorithm proposed in \cite{buis1996efficient} to do this, which we will henceforth refer to as \kronmatvec operation:
\[
	\mathbf{y} = \left(\mathbf{A}_1 \otimes \mathbf{A}_2 \otimes \cdots \otimes \mathbf{A}_n \right)\mathbf{x} 
\]

What we have achieved thus far is extending CORCONDIA to large and sparse data, assuming Frobenius norm. This assumption postulates that the underlying data distribution is Gaussian. However, recently \cite{chi2012tensors} showed that for sparse data that capture counts (e.g. number of messages exchanged), it is more beneficial to postulate a Poisson distribution, therefore using the KL-Divergence as a loss function. This has been more recently adopted in \cite{ho2014marble} showing very promising results in biomedical applications. Therefore, one natural direction, which we follow in the first part of the next section, is to extend CORCONDIA for this scenario.

\section{Proposed Methods}
\label{sec:method}

In exploratory data mining applications, the case is very frequently the following: we are given a piece of (usually very large) data that is of interest to a domain expert, and we are asked to identify regular and irregular patterns that are potentially useful to the expert who is providing the data. During this process, very often, the analysis is carried out in a completely unsupervised way, since ground truth and labels are either very expensive or impossible to obtain. In our context of tensor data mining, here is the problem at hand:

\begin{infproblem}
Given a tensor $\tensor{X}$ without ground truth or labelled data, how can we analyze it using the PARAFAC decomposition so that we can also:
\begin{enumerate}
	\item Determine automatically a good number of components for the decomposition
	\item Provide quality guarantees for the resulting decomposition
	\item Minimize human intervention and trial-and-error testing
\end{enumerate}
\end{infproblem}

In order to attack the above problem, first, in Section \ref{sec:method_corcondia} we describe how we can derive a fast and efficient metric of the quality of a decomposition, assuming the KL-Divergence. Finally, in \ref{sec:method_final}, we introduce \method, our unified algorithm for automatic tensor mining with minimal user intervention and quality characterization of the solution.

\label{sec:method}
\subsection{Quality Assessment with KL-Divergence}
\label{sec:method_corcondia}
As we saw in the description of CORCONDIA with Frobenius norm loss, its computation requires solving the least squares problem:
\[
\min_{\tensor{G}} 	\| vec\left( \tensor{X} \right) - \left( \mathbf{A}\otimes \mathbf{B} \otimes \mathbf{C} \right) vec \left( \tensor{G} \right)\|_F^2
\]

In the case of the \cpapr modelling, where the loss function is the KL-Divergence, the minimization problem that we need to solve is:
\begin{equation}
	\min_{\mathbf{x}} D_{KL}(\mathbf{y} || \mathbf{W}\mathbf{x})
	\label{eq:kl_regression}
\end{equation}
where in our case, $\mathbf{W} = \mathbf{A\otimes B \otimes C}$.

Unlike the Frobenius norm case, where the solution to the problem is the Least Squares estimate, in the KL-Divergence case, the problem does not have a closed form solution. Instead, iterative solutions apply. The most prominent approach to this problem is via an optimization technique called {\em Majorization-Minimization} (MM) or {\em Iterative Majorization} \cite{heiser1995convergent}. In a nutshell, in MM, given a function that is hard to minimize directly, we derive a ``majorizing'' function, which is always greater than the function to be minimized, except for a support point where it is equal; we minimize the majorizing function, and iteratively updated the support point using the minimizer of that function. This procedure converges to a local minimum. For the problem of Eq. \ref{eq:kl_regression}, \cite{fevotte2011algorithms} and subsequently \cite{chi2012tensors}, employ the following update rule for the problem, which is used iteratively until convergence to a stationary point.

\begin{equation}
	\mathbf{x}(j)^{(k)} = \mathbf{x}(j)^{(k-1)}  (  \frac{ \sum_i \mathbf{W}(i,k)  (  \frac{\mathbf{y}(j)}{\tilde{\mathbf{y}}}(j)^{(k-1)} )}{  \sum_i \mathbf{W}(i,j) } )
	\label{eq:update_rule_naive}
\end{equation}
where $\tilde{\mathbf{y}}^{(k-1)} = \mathbf{W}\mathbf{x}^{(k-1)}$, and $k$ denotes the $k$-th iteration index.

The above solution is generic for any structure of $\mathbf{W}$. Remember, however, that $\mathbf{W}$ has very specific Kronecker structure which we should exploit. Additionally, suppose that we have a $10^4 \times 10^4 \times 10^4$ tensor; then, the large dimension of $\mathbf{W}$ will be $10^{12}$. If we attempt to materialize, store, and use $\mathbf{W}$ throughout the algorithm, that can prove catastrophic to the algorithm's performance. We can exploit the Kronecker structure of $\mathbf{W}$ so that we break down Eq. \ref{eq:update_rule_naive} into pieces, each one which can be computed efficiently, given the structure of $\mathbf{W}$.
The first step is to decompose the expression of the numerator of Eq. \ref{eq:update_rule_naive}. In particular, we equivalently write
$
 \mathbf{x}^{(k)} = \mathbf{x}^{(k-1)} \ast \mathbf{z}_2
$
where 
$
	\mathbf{z}_2 = \mathbf{W}^T \mathbf{z}_1
$
and 
 $\mathbf{z}_1 = \mathbf{y}\oslash \tilde{\mathbf{y}}$. 
Due to the Kronecker structure of $\mathbf{W}$:
\[ 
 \mathbf{z}_2 = \kronmatvec( \{ \mathbf{A}^T, \mathbf{B}^T, \mathbf{C}^T  \},\mathbf{z}_1)
\]
Therefore, the update to $\mathbf{x}^{(k)} $ is efficiently calculated in the three above steps.
The normalization factor of the equation is equal to:
$
	\mathbf{s}(j) = \sum_i \mathbf{W}(i,j).
$
Given the Kronecker structure of $\mathbf{W}$ however, the following holds:
\begin{claim}
The row sum  of a Kronecker product matrix $\mathbf{A}\otimes\mathbf{B}$ can be rewritten as $\left( \sum_{i=1}^I \mathbf{A}(i,:) \right) \otimes \left( \sum_{j=1}^J \mathbf{B}(j,:)\right)$
\end{claim}
\begin{proof}
We can rewrite the row sums $\sum_{i=1}^I \mathbf{A}(i,:) =\mathbf{i}_I^T\mathbf{A}$ and $\sum_{j=1}^J \mathbf{B}(j,:) =\mathbf{i}_J^T\mathbf{B}$ where $\mathbf{i}_I$ and $\mathbf{i}_J$ are all-ones column vectors of size $I$ and $J$ respectively. For the Kronecker product of the row sums and by using properties of the Kronecker product, and calling $\mathbf{A\otimes B} = \mathbf{W}$ we have
\[
 \left( \mathbf{i}_I^T\mathbf{A} \right) \otimes \left( \mathbf{i}_J^T\mathbf{B} \right) = \left(\mathbf{i}_I \otimes \mathbf{i}_J \right)^T \left( \mathbf{A\otimes B} \right) =  \mathbf{i}_{IJ}^T\mathbf{W} = \sum_{i=1}^{IJ}\mathbf{W}(i,:)
\]
which concludes the proof.
\end{proof}
Thus, $
\mathbf{s} = \left( \sum_i \mathbf{A}(i,:) \right)  \otimes \left( \sum_j \mathbf{B}(j,:) \right) \otimes \left(\sum_n \mathbf{C}(n,:) \right).
$

Putting everything together, we end up with Algorithm \ref{alg:efficient_kl_regression} which is an efficient solution to the minimization problem of Equation \ref{eq:kl_regression}. As in the naive case, we also use Iterative Majorization in the efficient algorithm; we iterate updating $\mathbf{x}^{(k)}$ until we converge to a local optimum. Finally, Algorithm \ref{alg:efficient_kl} shows the steps to compute CORCONDIA under KL-Divergence efficiently.
\begin{center}
\begin{algorithm} [!htp]
\small
\caption{Efficient Quality Assesment with KL-Divergence loss} \label{alg:efficient_kl}
\begin{algorithmic}[1]
\REQUIRE Tensor $\tensor{X}$ and \cpapr factor matrices $\mathbf{A,B,C}$.
\ENSURE CORCONDIA diagnostic $c$.
\STATE some more stuff
\STATE some stuff here
\STATE $\displaystyle{ c = 100\left(1 - \frac{\sum_{i=1}^{F} \sum_{j=1}^{F} \sum_{k=1}^{F} \left(\tensor{G}(i,j,k) - \tensor{I}(i,j,k)\right)^2 }{F}  \right)}$
\end{algorithmic}
\end{algorithm}
\end{center}
\begin{center}
\begin{algorithm} [!htp]
\small
\caption{Efficient Majorization Minimization for KL-Divergence Regression} \label{alg:efficient_kl_regression}
\begin{algorithmic}[1]
\REQUIRE Vector $\mathbf{y}$ and matrices $\mathbf{A,B,C}$.
\ENSURE Vector $\mathbf{x}$
\STATE Initialize $\mathbf{x}^{(0)}$ randomly
\STATE $\tilde{\mathbf{y}} = \kronmatvec(\{ \mathbf{A, B, C} \},\mathbf{x}^{(0)})$
\STATE $\mathbf{s} = \left( \sum_i \mathbf{A}(i,:) \right)  \otimes \left( \sum_j \mathbf{B}(j,:) \right) \otimes \left(\sum_n \mathbf{C}(n,:) \right)$
\STATE Start loop:
\STATE $\mathbf{z}_1 = \mathbf{y}\oslash \tilde{\mathbf{y}}$
\STATE $\mathbf{z}_2 = \kronmatvec( \{ \mathbf{A}^T, \mathbf{B}^T, \mathbf{C}^T  \},\mathbf{z}_1)$
\STATE $\mathbf{x}^{(k)} = \mathbf{x}^{(k-1)} \ast \mathbf{z}_2$
\STATE $\tilde{\mathbf{y}} = \kronmatvec(\{ \mathbf{A, B, C} \},\mathbf{x}^{(k)})$
\STATE End loop
\STATE Normalize $\mathbf{x}^{(k)} $ using $\mathbf{s}$
\end{algorithmic}
\end{algorithm}
\end{center}

\subsection{\methodplain: Automated Unsupervised Tensor Mining}
\label{sec:method_final}

At this stage, we have the tools we need in order to design an automated tensor mining algorithm that minimizes human intervention and provides quality characterization of the solution. We call our proposed method \method, and we view this as a step towards making tensor mining a fully automated tool, used as a black box by academic and industrial practicioners.

\method is a two step algorithm, where we first search through the solution space and at the second step, we automatically select a good solution based on its quality and the number of components it offers.
A sketch of \method follows, and is also outlined in Algorithm \ref{alg:auto_tensor}.
\paragraph{Solution Search}
The user provides a data tensor, as well as a maximum rank that reflects the budget that she is willing to devote to \method's search. We neither have nor require any prior knowledge whether the tensor is highly sparse, or dense, contains real values or counts, hinting whether we should use, say, \cpnmu postulating Frobenius norm loss, or \cpapr postulating KL-Divergence loss.

Fortunately, our work in this paper, as well as our previous work \cite{papalexakis2015fastcorcondia} has equipped us with tools for handling all of the above cases. Thus, we follow a data-driven approach, where we let the data show us whether using \cpnmu or \cpapr is capturing better structure.
For a grid of values for the decomposition rank (bounded by the user provided maximum rank), we run both \cpnmu and \cpapr, and we record the quality of the result as measured by the CORCONDIA diagnostic into vectors $\mathbf{c}_{Fro}$ and $\mathbf{c}_{KL}$ (using the algorithm in \cite{papalexakis2015fastcorcondia} and Algorithm \ref{alg:efficient_kl} respectively), truncating negative values to zero.

\paragraph{Result Selection}
At this point, for both \cpnmu and \cpapr we have points in two dimensional space $(F_i,c_i)$, reflecting the quality and the corresponding number of components. 
Informally, our problem here is the following:
\begin{infproblem}
Given points $(F_i,c_i)$ we need to find one that maximizes the quality of the decomposition, as well as finding as many hidden components in the data as possible. 
\end{infproblem}
Intuitively, we are seeking a decomposition that discovers as many latent components as possible, without sacrificing the quality of those components. 
Essentially, we have a multi-objective optimization problem, where we need to maximize both $c_i$ and $F_i$. However, if we, say, get the Pareto front of those points (i.e. the subset of all non-dominated points), we end up with a family of solutions without a clear guideline on how to select one. We propose to use the following, effective, two-step maximization algorithm that gives an intuitive {\em data-driven} solution:
\begin{itemize}
	\item {\bf Max c step}: Given vector $\mathbf{c}$, run 2-means clustering on its values. This will essentially divide the vector into a set of good/high values and a set of low/bad ones. If we call $m_1 , m_2$ the means of the two clusters, then we select the cluster index that corresponds to the maximum between $m_1$ and $m_2$.
	\item {\bf Max F step}: Given the cluster of points with maximum mean, we select the point that maximizes the value of $F$. We call this point $(F^*,c^*)$.
\end{itemize}
Another alternative is to formally define a function of $c,F$ that we wish to maximize, and select the maximum via enumeration. Coming up with the particular function to maximize, considering the intuitive objective of maximizing the number of components that we can extract with reasonably high quality ($c$), is a hard problem, and we risk biasing the selection with a specific choice of a function. Nevertheless, an example such function can be $g(c,F) = logc logF$ for $c>0$, and $g(0,F) = 0$; this function essentially measures the area of the rectangle formed by the lines connecting $(F,c)$ with the axes (in the log-log space) and intuitively seeks to find a good compromise between maximizing $F$ and $c$. This function performs closely to the proposed data-driven approach and we defer a detailed discussion and investigation to future work.

\hide{
Figure \ref{fig:pareto_heuristic} shows pictorially the output of this two-step algorithm for a set of points taken from a real dataset.

\begin{figure}[!ht]
	\begin{center}
		\includegraphics[width = 0.4\textwidth]{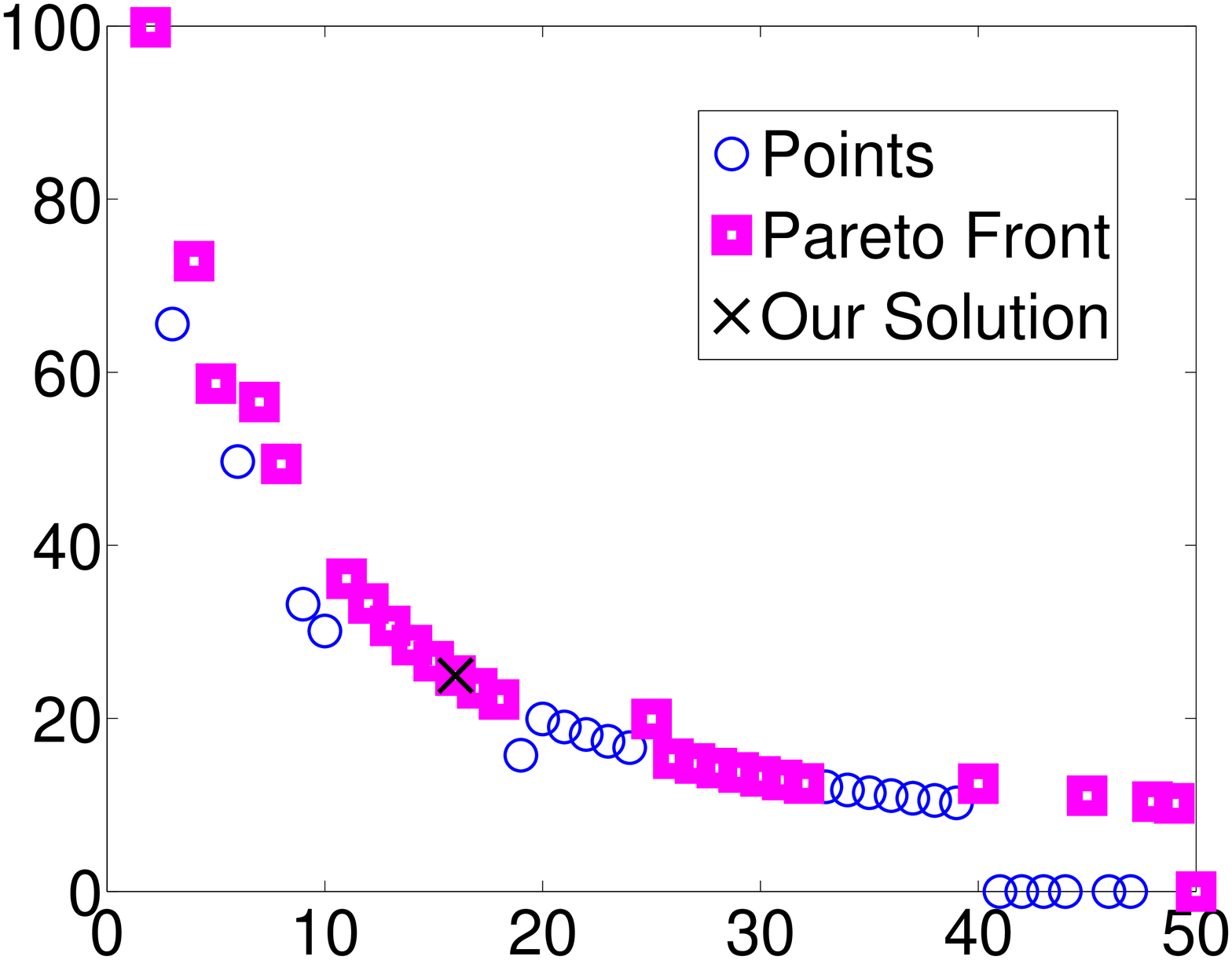}
	\end{center}
	\caption{Example of choosing a good point}
	\label{fig:pareto_heuristic}
\end{figure}
}
After choosing the ``best'' points $(F_{Fro}^*,c_{Fro}^*)$ and $(F_{KL}^*,c_{KL}^*)$, at the final step of \method, we have to select between the results of \cpnmu and \cpapr. In order do so, we can use the following strategies:
 \begin{enumerate}
 	\item  Calculate $s_{Fro} = \displaystyle{\sum_f \mathbf{c}_{Fro}(f)}$ and $s_{KL} = \displaystyle{\sum_f \mathbf{c}_{KL}(f)}$, and select the method that gives the largest sum. The intuition behind this data-driven strategy is choosing the loss function that is able to discover results with higher quality on aggregate, for more potential ranks.
 	\item Select the results that produce the maximum value between $c_{Fro}^*$ and $c_{KL}^*$. This strategy is conservative and aims for the highest quality of results, possibly to the expense of components of lesser quality that could still be acceptable for exploratory analysis.
 	\item Select the results that produce the maximum value between $F_{Fro}^*$ and $F_{KL}^*$. Contrary to the previous strategy, this one is more aggressive, aiming for the highest number of components that can be extracted with acceptable quality.
 \end{enumerate}
Empirically, the last strategy seems to give better results, however they all perform very closely in synthetic data. Particular choice of strategy depends on the application needs, e.g. if quality of the components is imperative to be high, then strategy 2 should be preferred over strategy 3.

\begin{center}
\begin{algorithm} [!htp]
\small
\caption{\method: Automatic Unsupervised Tensor Mining} \label{alg:auto_tensor}
\begin{algorithmic}[1]
\REQUIRE Tensor $\tensor{X}$ and maximum budget for component search $F$
\ENSURE PARAFAC decomposition $\mathbf{A,B,C}$ of $\tensor{X}$ and corresponding quality metric $c^*$.
\FOR{$f = 2\cdots F$}
\STATE Run \cpnmu for $f$ components. Update $\mathbf{c}_{Fro}(f)$ with the result of Algorithm in \cite{papalexakis2015fastcorcondia}.
\STATE Run \cpapr for $f$ components. Update $\mathbf{c}_{KL}(f)$ with the result of Algorithm \ref{alg:efficient_kl}.
\ENDFOR
\STATE Find $(F_{Fro}^*,c_{Fro}^*)$ and $(F_{KL}^*,c_{KL}^*)$ using the two-step maximization as described in the text.
\STATE Choose between \cpnmu and \cpapr using one of the three strategies described in the text.
\STATE Output the chosen $c^*$ and the corresponding decomposition.
\end{algorithmic}
\end{algorithm}
\end{center}
We point out that lines 1-5 of Algorithm \ref{alg:auto_tensor} are embarrassingly parallel.
Finally, it is important to note that \method not only seeks to find a good number of components for the decomposition, combining the best of both worlds of \cpnmu and \cpapr, but furthermore is able to provide quality assessment for the decomposition: if for a given $F_{max}$ none of the solutions that \method sifts through yields a satisfactory result, the user will be able to tell because of the very low (or zero in extreme cases) $c^*$ value. 
\section{Experimental Evaluation}
\label{sec:experiments}
We implemented \method in Matlab, using the Tensor Toolbox \cite{tensortoolbox}, which provides efficient manipulation and storage of sparse tensors. We make our code publicly available\footnote{Download our code at \\ \codeurl}. The online version of our code contains a test case that uses the same code that we used for the following evaluation.
All experiments were run on a workstation with 4 Intel(R) Xeon(R) E7- 8837 and 1TB of RAM.

\subsection{Evaluation on Synthetic Data}
\label{sec:synthetic}
In this section, we empirically measure \method's ability to uncover the true number of components hidden in a tensor. The experimental setting is as follows: We create synthetic tensors of size $50\times50\times50$, using the function \texttt{create\_problem} of the Tensor Toolbox for Matlab as a standardized means of generating synthetic tensors, we create two different test cases: 1) sparse factors, with total number of non-zeros equal to 500, and 2) dense factors.
In both cases, we generate random factors with integer values.
We generate these three test cases for true rank $F_o$ ranging from 2-5. For both test cases, we distinguish a noisy case (where Gaussian noise with variance $0.1$ is by default added by \texttt{create\_problem}) and a noiseless case.

We compare \method against three baselines:
\begin{itemize}\setlength{\itemsep}{0.0mm}
	\item {\bf Baseline 1}: A Bayesian tensor decomposition approach, as introduced very recently in \cite{zhaobayesian} which automatically determines the rank.
	\item {\bf Baseline 2}: This is a very simple heuristic approach where, for a grid of values for the rank, we run \cpnmu and record the Frobenius norm loss for each solution. If for two consecutive iterations the loss does not improve more than a small positive number $\epsilon$ (set to $10^{-6}$ here), we declare as output the result of the previous iteration.
	\item {\bf Baseline 3}: Same as Baseline 2 with sole difference being that we use \cpapr and accordingly instead of the Frobenius norm reconstruction error, we measure the log-likelihood, and we stop when it stops improving more than $\epsilon$. We expect Baseline 3 to be more effective than Baseline 2 in sparse data, due to the more delicate and effective treatment of sparse, count data by \cpapr.
\end{itemize}

\method as well as Baselines 2 \& 3 require a maximum bound $F_{max}$ on the rank; for fairness, we set $F_{max}=2F_o$ for all three methods.
In Figures \ref{fig:synthetic_errors} and \ref{fig:synthetic_errors_noiseless} we show the results for both test cases, for noisy and noiseless data respectively. The error is measured as $|F_{est} - F_o|$ where $F_{est}$ is the estimated number of components by each method. Due to the randomized nature of the synthetic data generation, we ran 1000 iterations and we show the average results. 
In the noisy case (Fig. \ref{fig:synthetic_errors}) we observe that in both scenarios and for all chosen ranks, \method outperforms the baselines, having lower error. 
In the noiseless case of Fig. \ref{fig:synthetic_errors_noiseless}, we observe consistent behavior, with all methods experiencing a small boost to their performance, due to the absence of noise.
We calculated statistical significance of our results ($p<0.01$) using a two-sided sign test.

Overall, we conclude that \method largely outperforms the baselines. The problem at hand is an extremely hard one, and we are not expecting any {\em tractable} method to solve it perfectly. Thus, the results we obtain here are very encouraging.

\begin{figure}[!htf]
\begin{center}
	\subfigure[Sparse]{ \includegraphics[width = 0.23\textwidth]{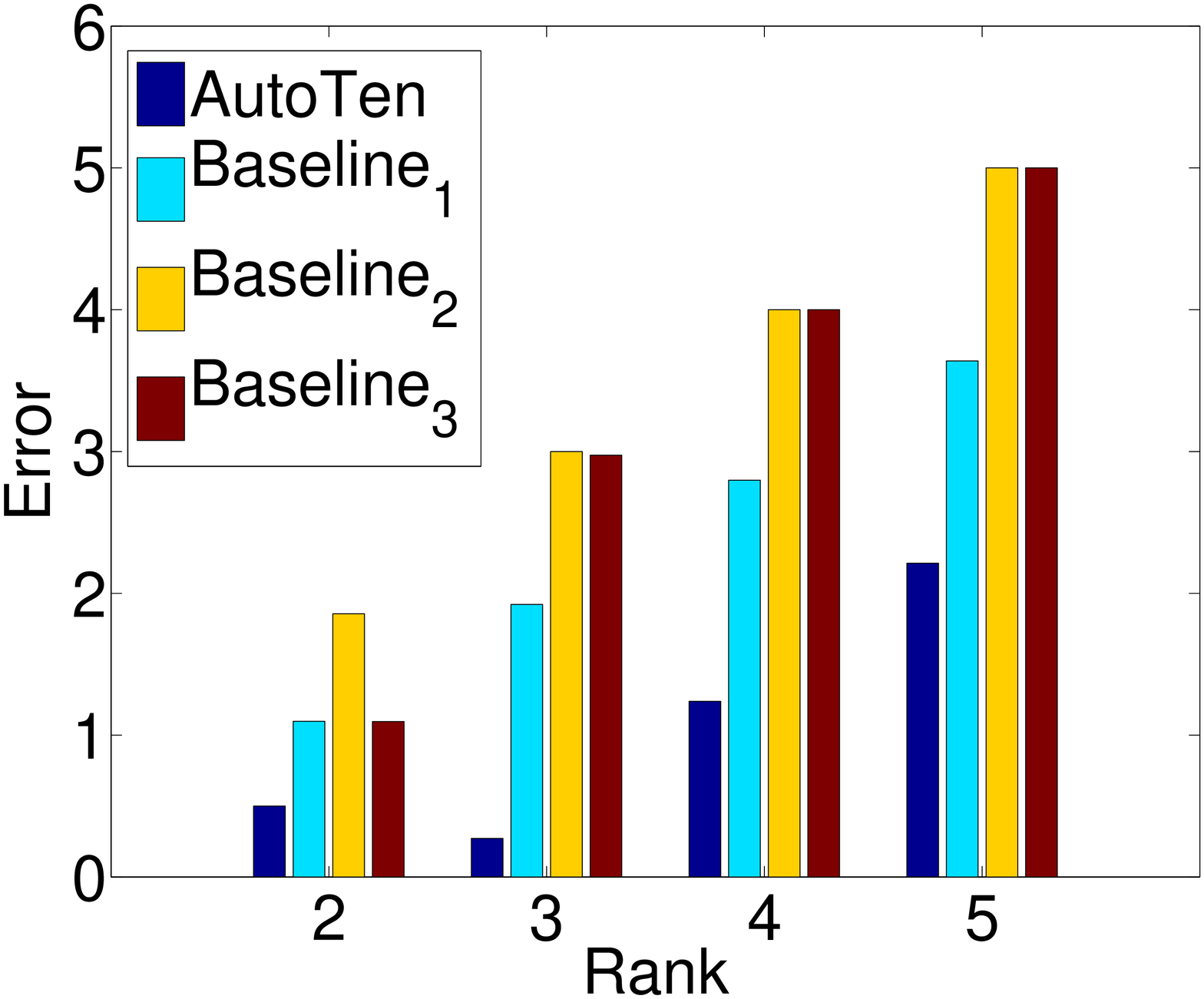}}
	\subfigure[Dense]{ \includegraphics[width = 0.23\textwidth]{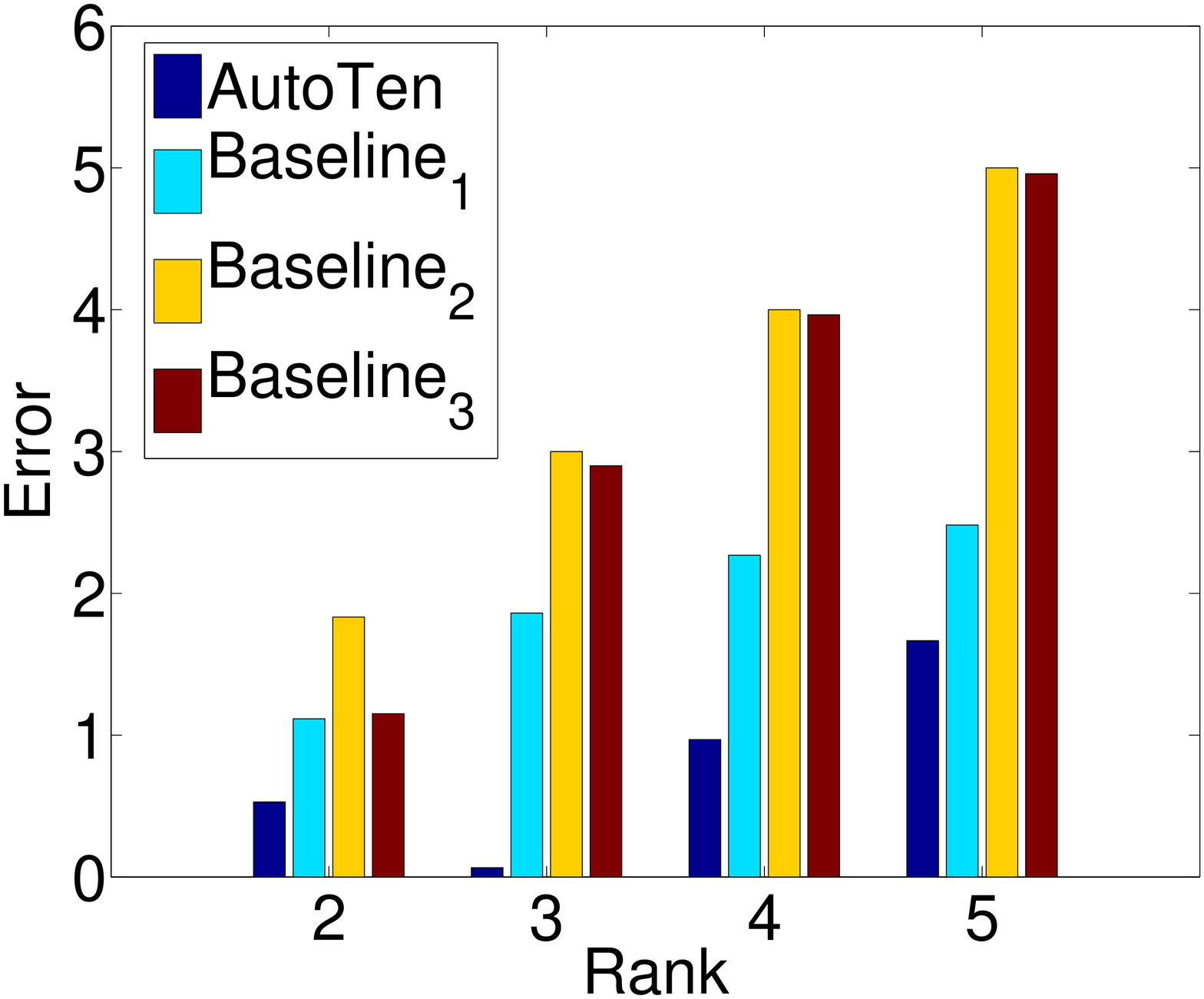}}
	\caption{Error for \method and the baselines, for noisy synthetic data.}
	\label{fig:synthetic_errors}
\end{center}
\end{figure}
\begin{figure}[!htf]
\begin{center}
	\subfigure[Sparse]{ \includegraphics[width = 0.23\textwidth]{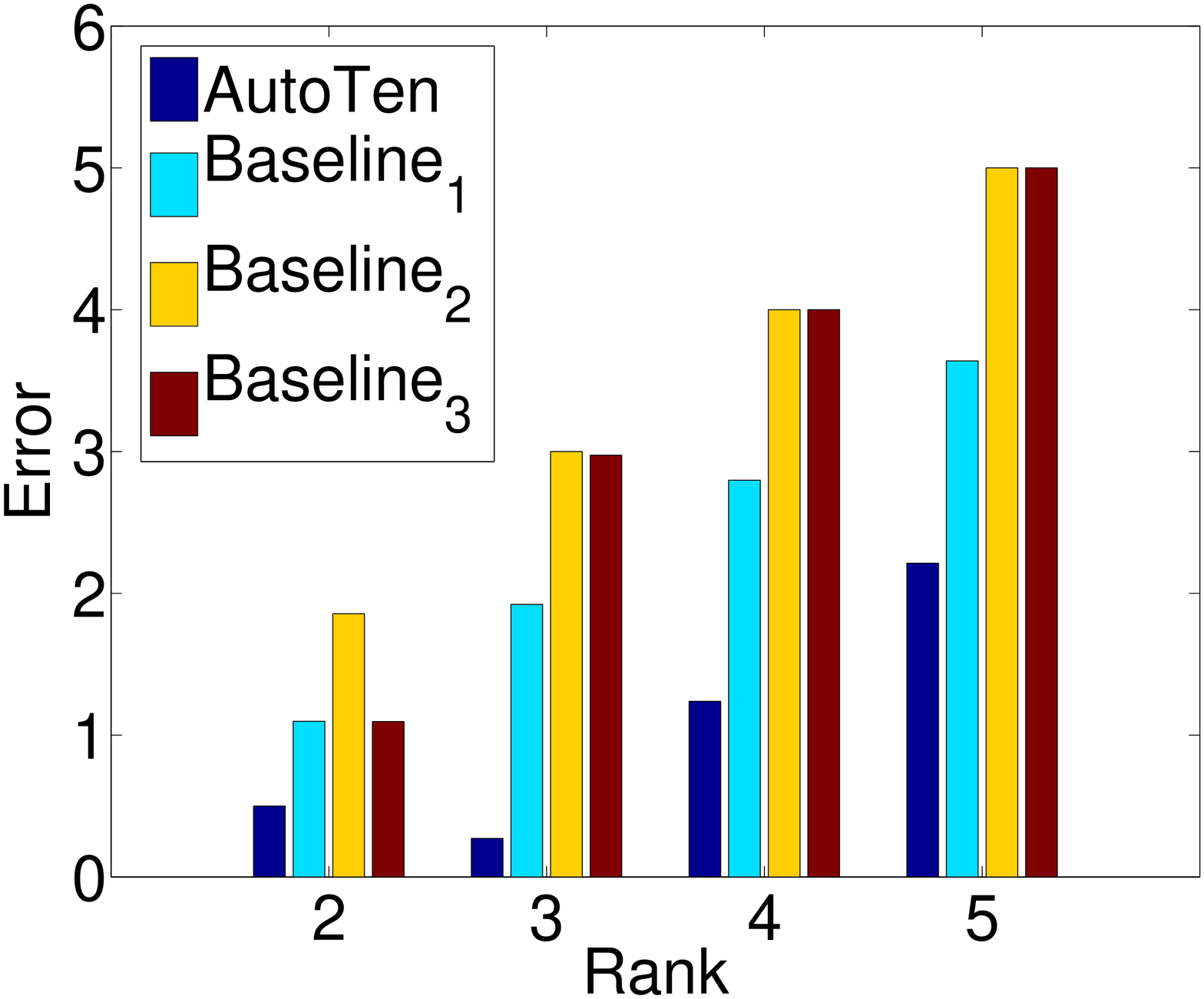}}
	\subfigure[Dense]{ \includegraphics[width = 0.23\textwidth]{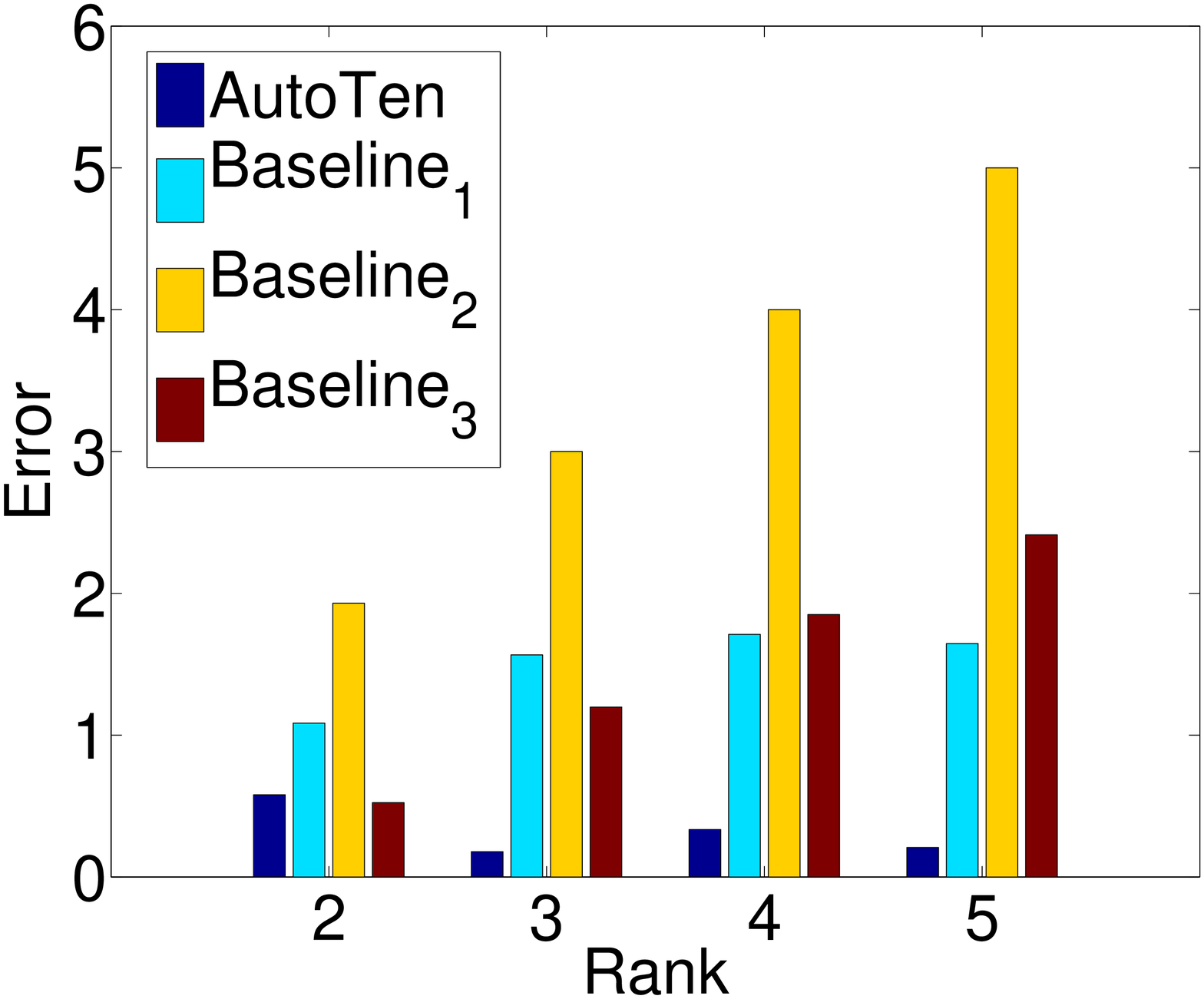}}
	\caption{Error for \method and the baselines, for noiseless synthetic data.}
	\label{fig:synthetic_errors_noiseless}
\end{center}
\end{figure}


\section{Data Mining Case Study}
\label{fig:case_study}
After establishing that \method is able to perform well in a control, synthetically generated setting, the next step is to see its results ``in the wild''. To that end, we are conducting two case studies. Section \ref{sec:study1} takes 10 diverse real datasets shown in Table \ref{tab:datasets} and investigates their rank structure. In Section \ref{sec:study2} we apply \method to two of the datasets of Table \ref{tab:datasets} and we analyze the results, as part of an exploratory data mining study.
 \begin{table*}[!ht]
 \small
 \begin{center}
    \caption{Datasets analyzed \label{tab:datasets}}
{ 
  \rowcolors{1}{lightgray}{white}
  \begin{tabular}{llllll}
    & \textbf{Name} & \textbf{Description} & \textbf{Dimensions} & {\bf Number of nonzeros}   \\
    & \enron & (sender, recipient, month) & $186\times186\times44$ & 9838 \\ \hiderowcolors
    & \reality\cite{eagle2009inferring} & (person, person, means of communication)  &$88\times88 \times 4$     & 5022 \\
	& \facebook \cite{viswanath-2009-activity} & (wall owner, poster, day) &  $63891\times  63890   \times 1847$   & 737778 \\
	& \taxi\cite{yuan2011driving,Wang2014travel}  &  (latitude, longitude,minute) &   $100\times100\times 9617$  & 17762489 \\
	& \dblp \cite{papalexakis2013more} &  (paper, paper, view)&$7317\times7317\times3$     &  274106\\
	& \netflix & (movie, user, date) &   $17770\times252474\times88$  &  50244707 \\
	& \amazonco\cite{snapnets} &  (product, product, product group)& $256\times256\times5$    & 5726 \\
	& \amazonmeta\cite{snapnets} &  (product, customer, product group) &   $10000\times263011\times5$  &  441301\\
	& \yelp &  (user, business, term) &$43872\times11536\times10000$  &  10009860\\
	& \airport &  (airport, airport, airline) &  $9135\times9135\times19305$    & 58443 \\
  \end{tabular}
  }
 \end{center}
\end{table*}

\subsection{Rank Structure of Real Datasets}
\label{sec:study1}
Since exploration of the rank structure of a dataset, using the CORCONDIA diagnostic, is an integral part of \method, we deem necessary to dive deeper into that process. 
In this case study we are analyzing the rank structure of 10 real datasets, as captured by CORCONDIA with Frobenius norm loss (using our algorithm from \cite{papalexakis2015fastcorcondia}, as well as CORCONDIA with KL-Divergence loss (introduced here).
Most of the datasets we use are publicly available and can be obtained by either following the link within the original work that introduced them, or (whenever applicable) a direct link.
\enron\footnote{\url{http://www.cs.cmu.edu/~enron/}} is a social network dataset, recording the number of emails exchanged between employees of the company for a period of time, during the company crisis. \reality \cite{eagle2009inferring} is a multi-view social network dataset, recording relations between MIT students (who calls whom, who messages whom, who is close to whom and so on). \facebook \cite{viswanath-2009-activity} is a time evolving snapshot of Facebook, recording people posting on other peoples' Walls. \taxi\footnote{\small\url{http://research.microsoft.com/apps/pubs/?id=152883}} is a dataset of taxi trajectories in Beijing; we discretize latitude and longitude to a $100\times 100$ grid. \dblp is a dataset recording which researched published what paper under three different views (first view shows co-authorship, second view shows citation, and third view shows whether two authors share at least three keywords in their title or abstract of their papers). \netflix comes from the Netflix prize dataset and records movie ratings by users over time.  \amazonco data records items bought together, and the category of the first of the two products. \amazonmeta records customers who reviewed a product, and the corresponding product category. \yelp contains reviews of Yelp users for various businesses (from the data challenge\footnote{\url{https://www.yelp.com/dataset_challenge/dataset}}). Finally, \airport\footnote{\url{http://openflights.org/data.html}} contains records of flights between different airports, and the operating airline.

\begin{figure*}[!ht]
	\begin{center}
		\subfigure[]{ \includegraphics[width = \subfigurewidth]{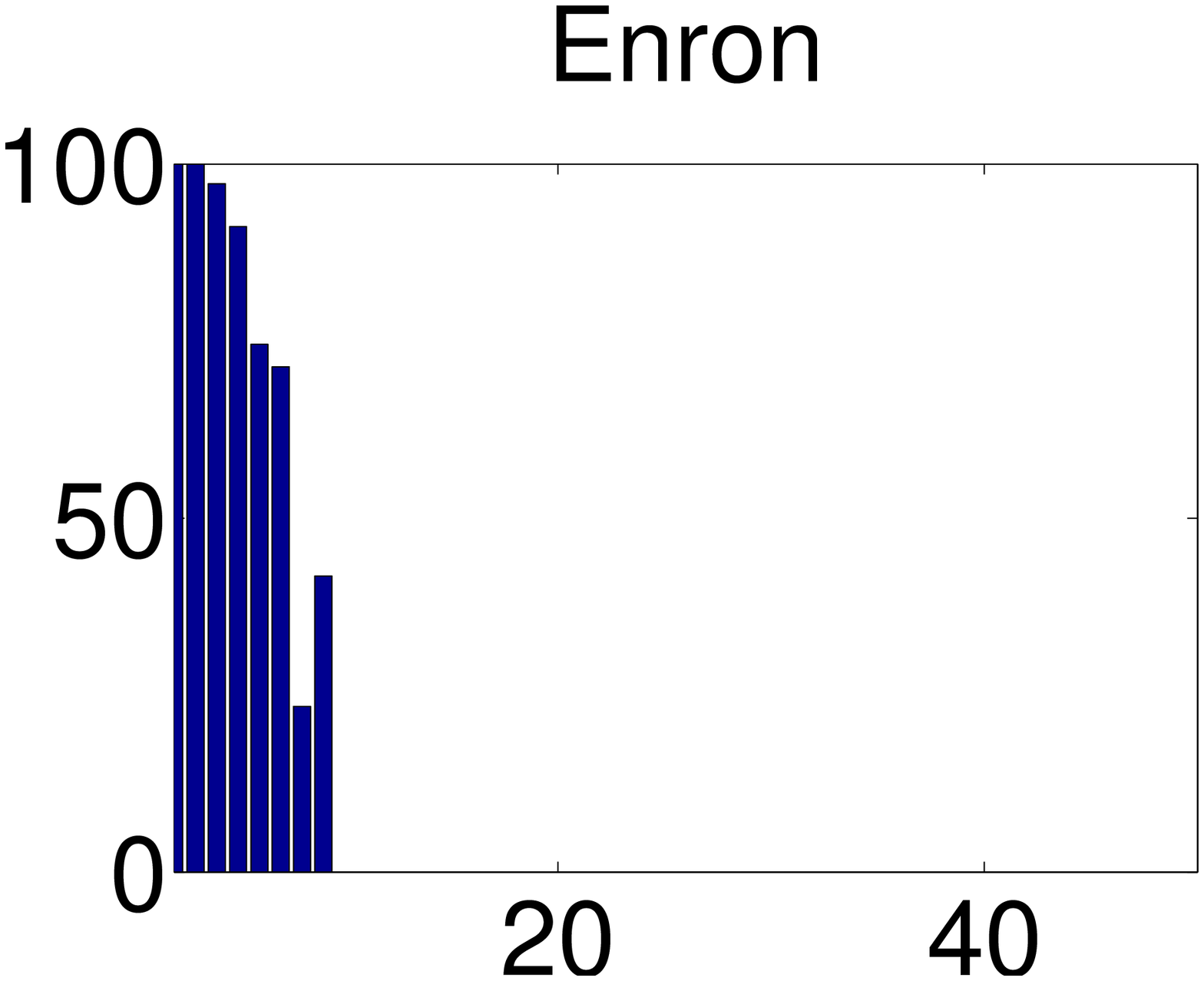} }
		\subfigure[]{ \includegraphics[width = \subfigurewidth]{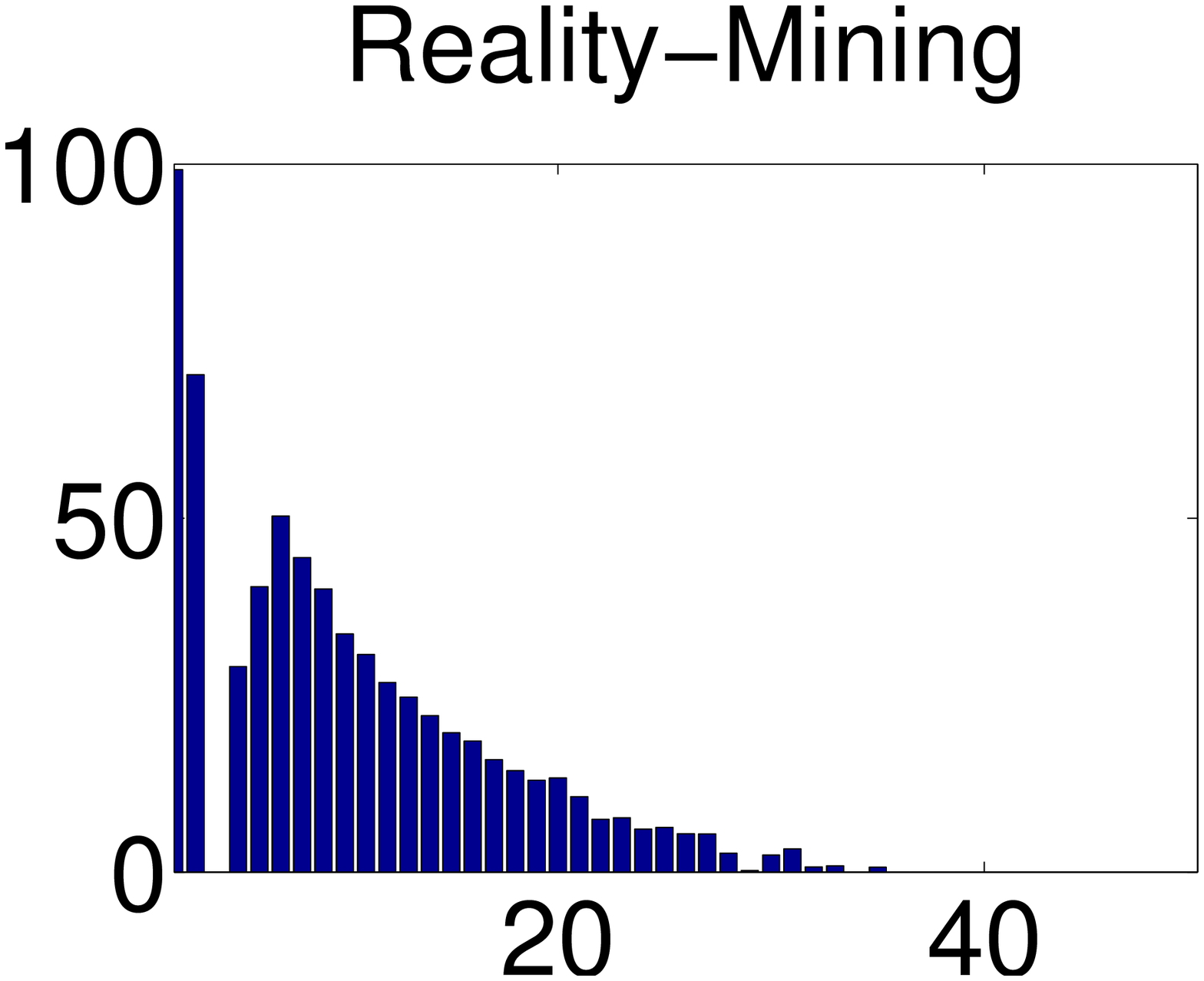} }
		\subfigure[]{ \includegraphics[width = \subfigurewidth]{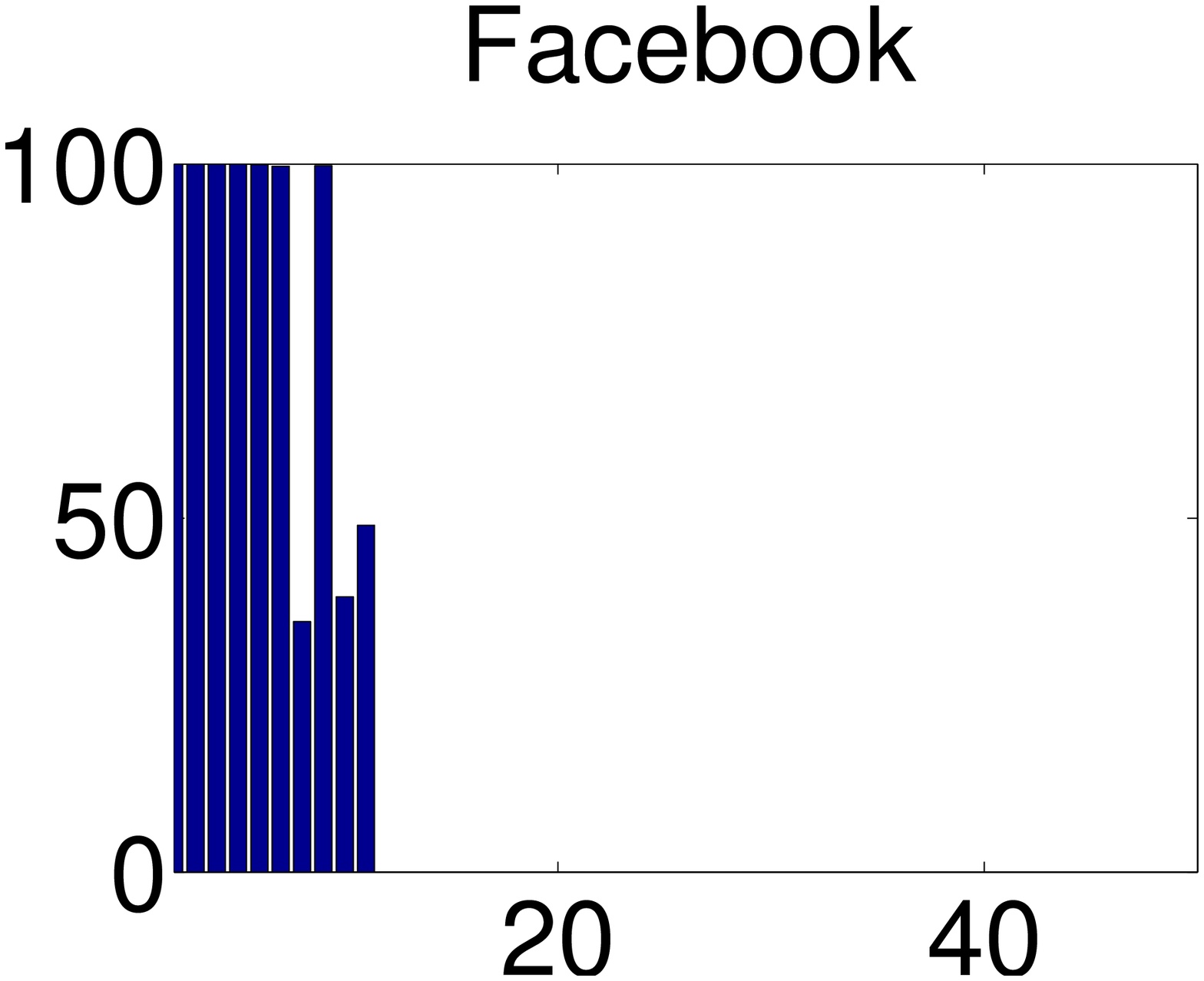} }
		\subfigure[]{ \includegraphics[width = \subfigurewidth]{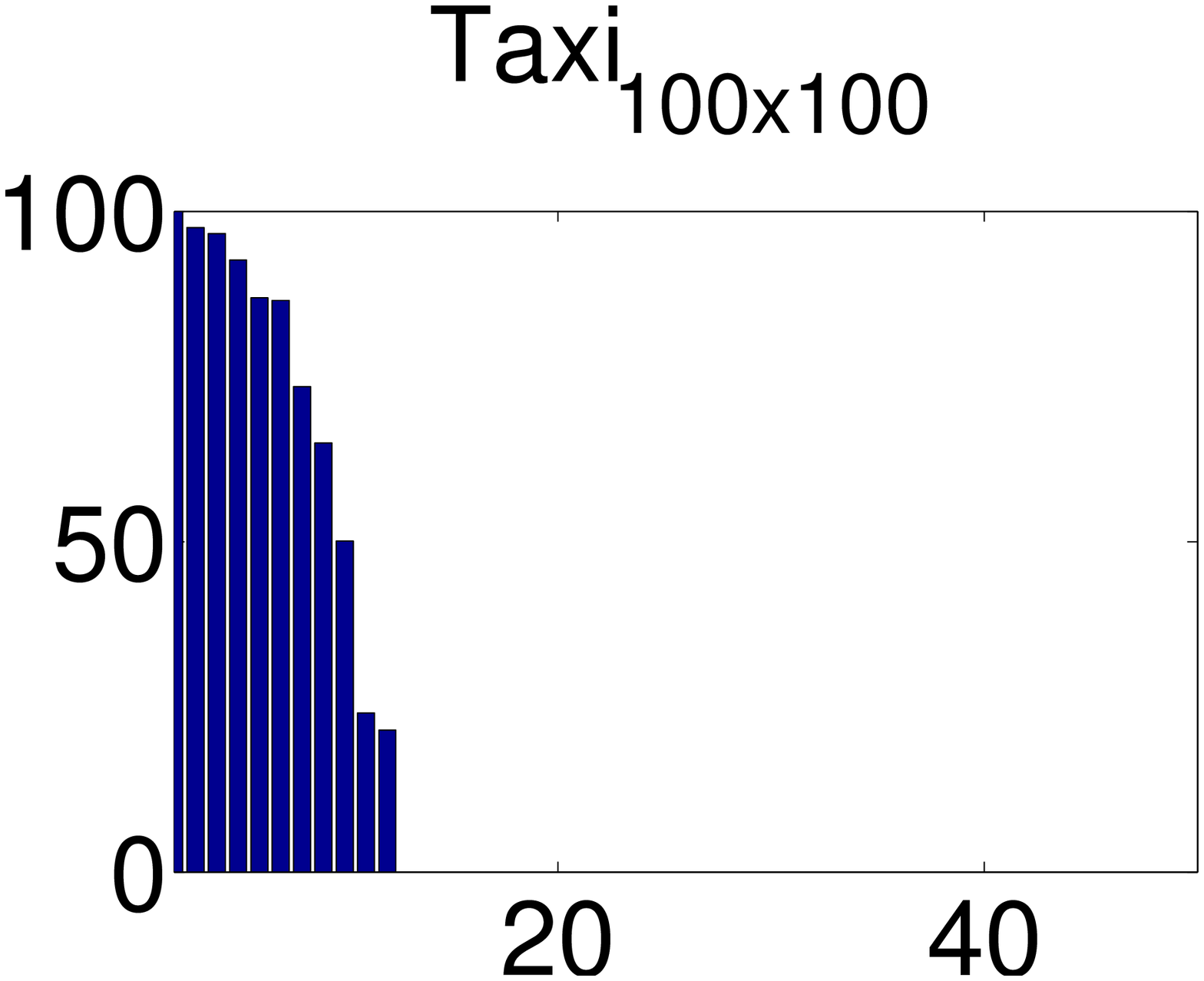} }
		\subfigure[]{ \includegraphics[width = \subfigurewidth]{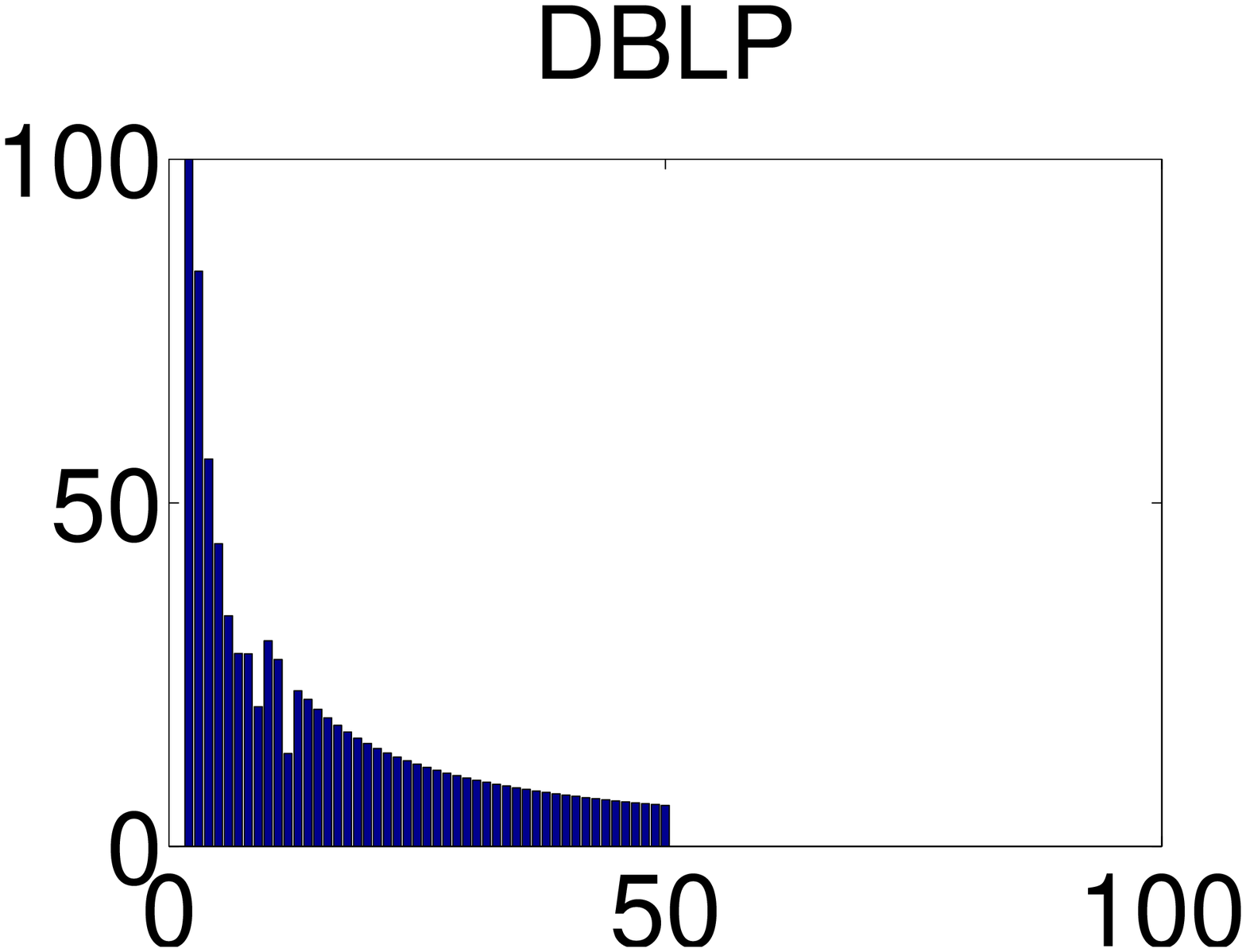} }
		\subfigure[]{ \includegraphics[width = \subfigurewidth]{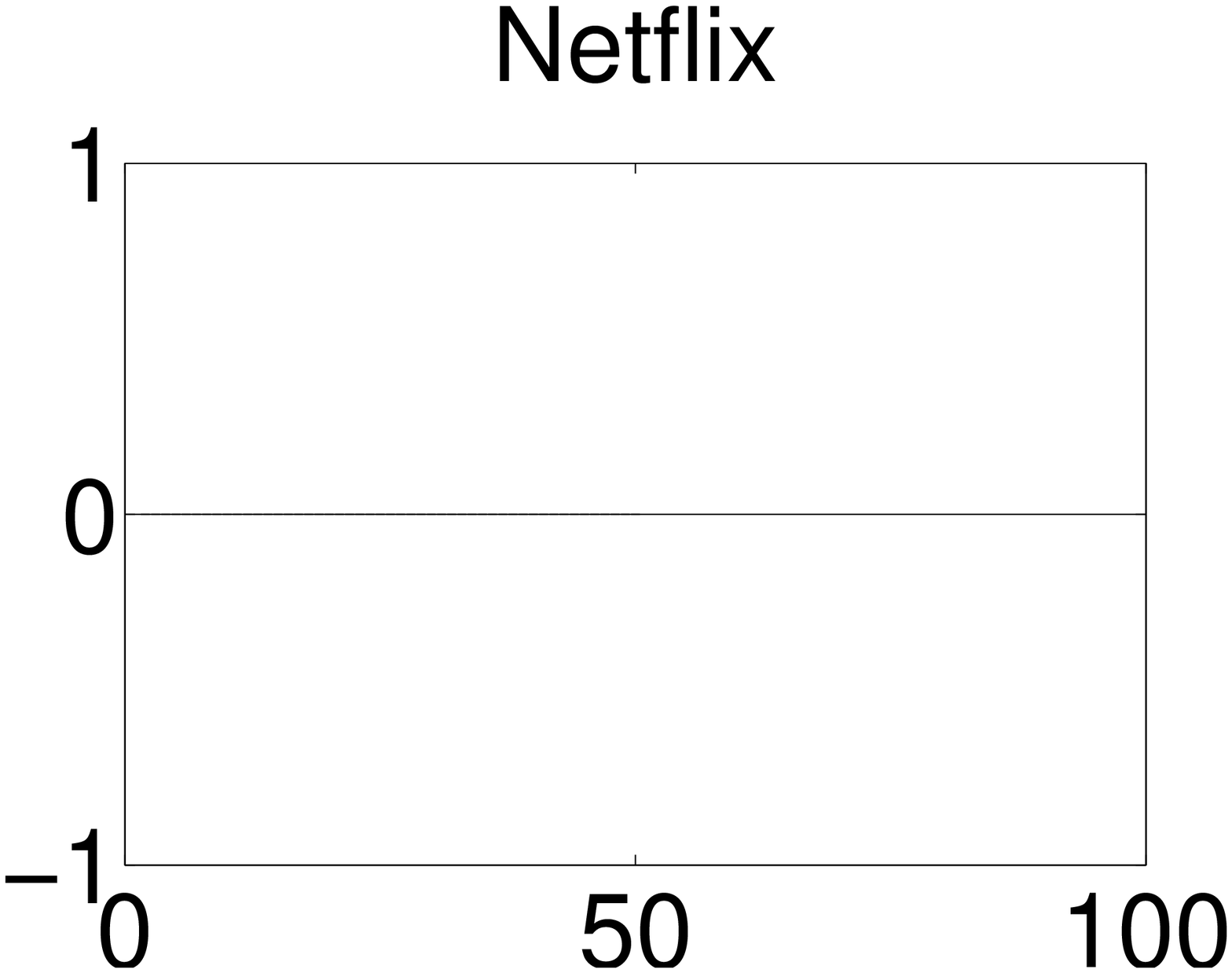} }
		\subfigure[]{ \includegraphics[width = \subfigurewidth]{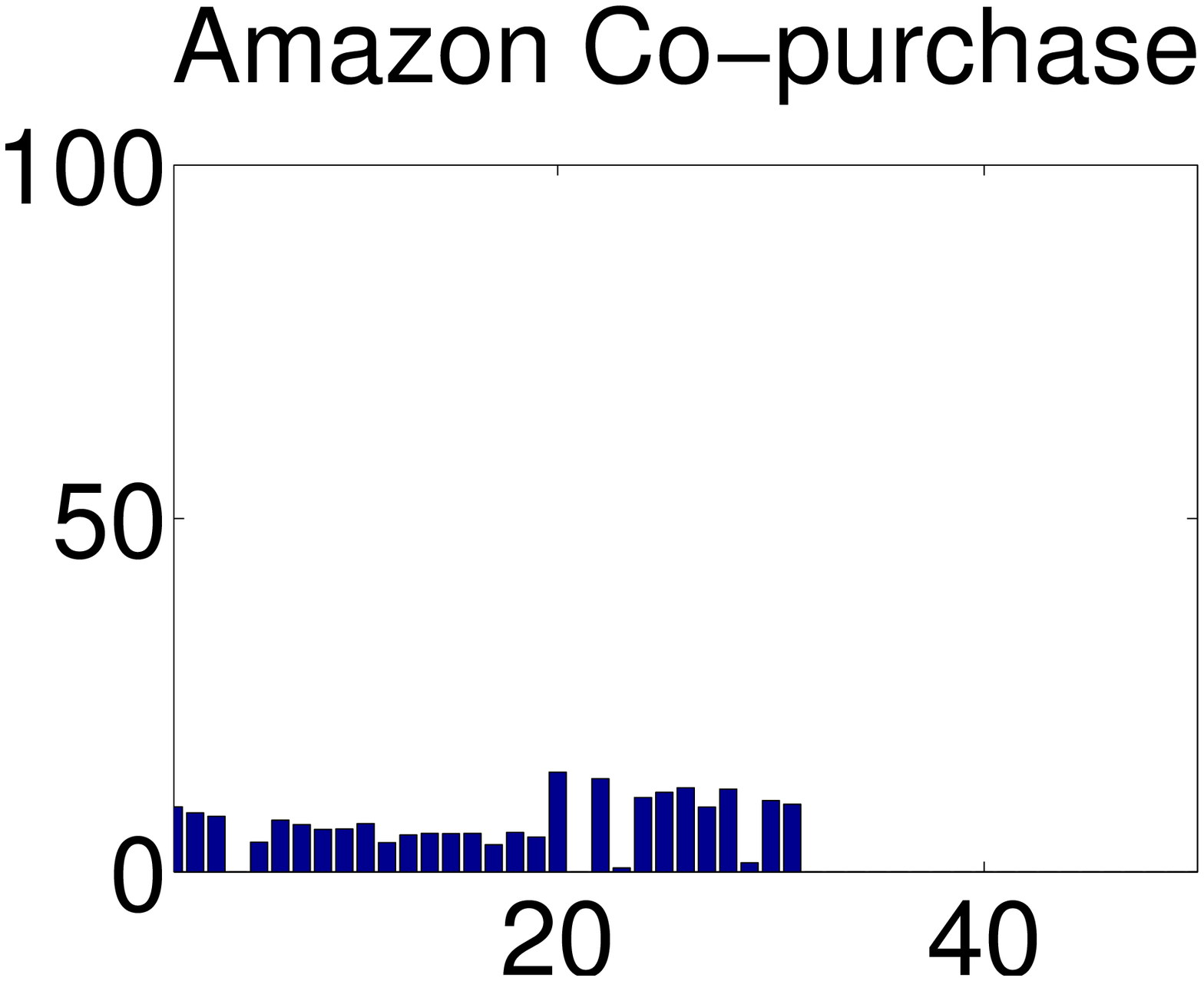} }
		\subfigure[]{ \includegraphics[width = \subfigurewidth]{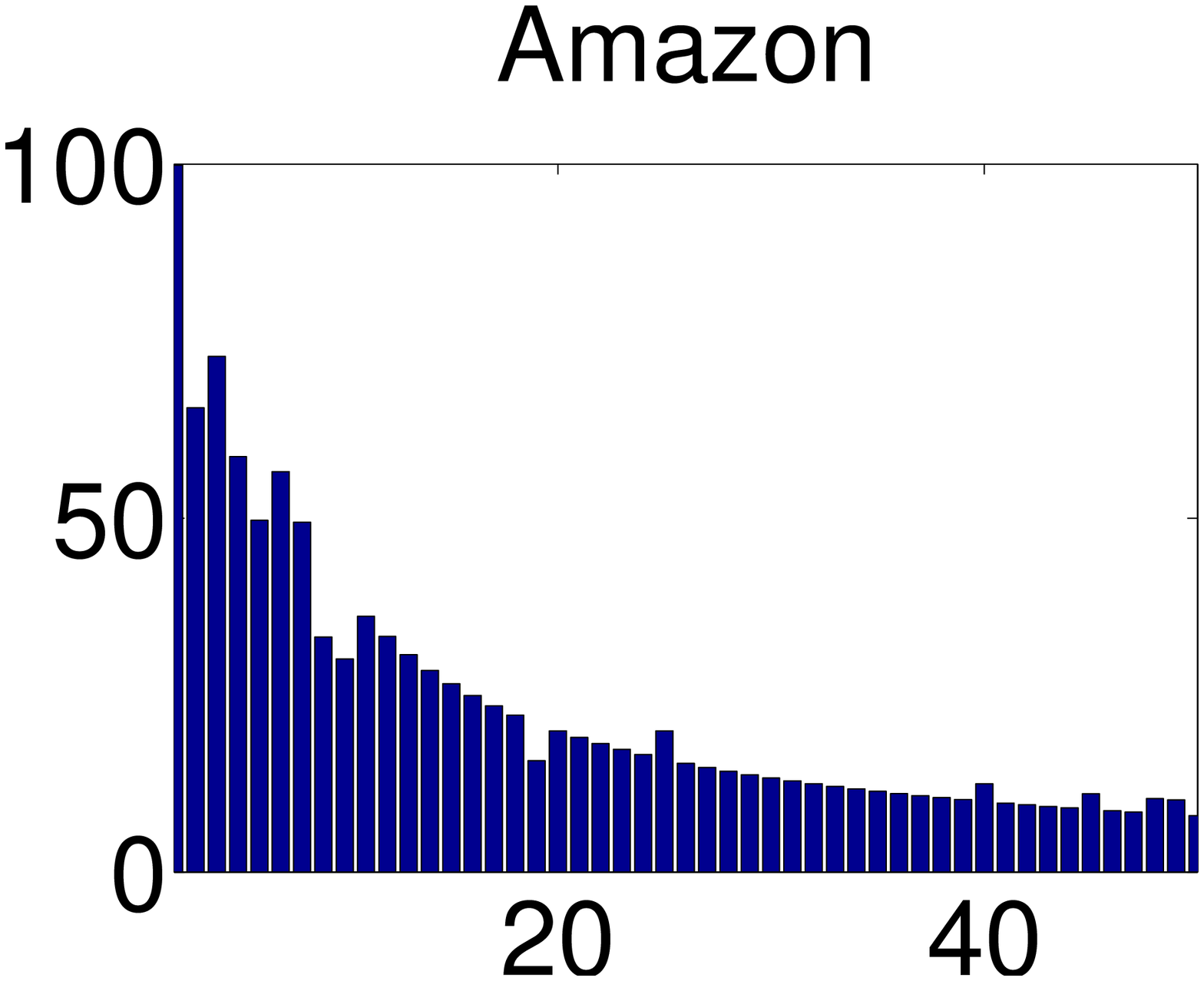} }
		\subfigure[]{ \includegraphics[width = \subfigurewidth]{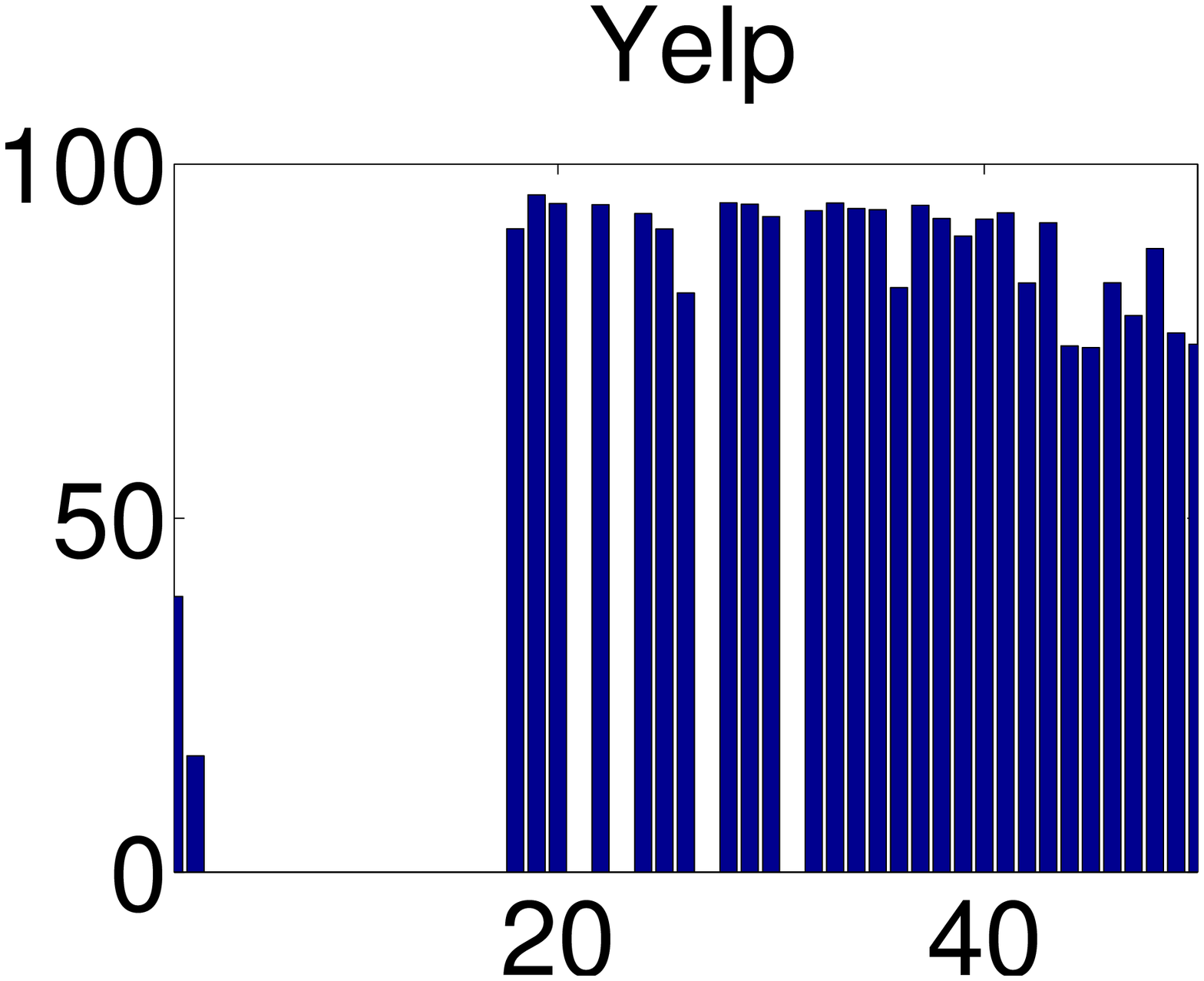} }
		\subfigure[]{ \includegraphics[width = \subfigurewidth]{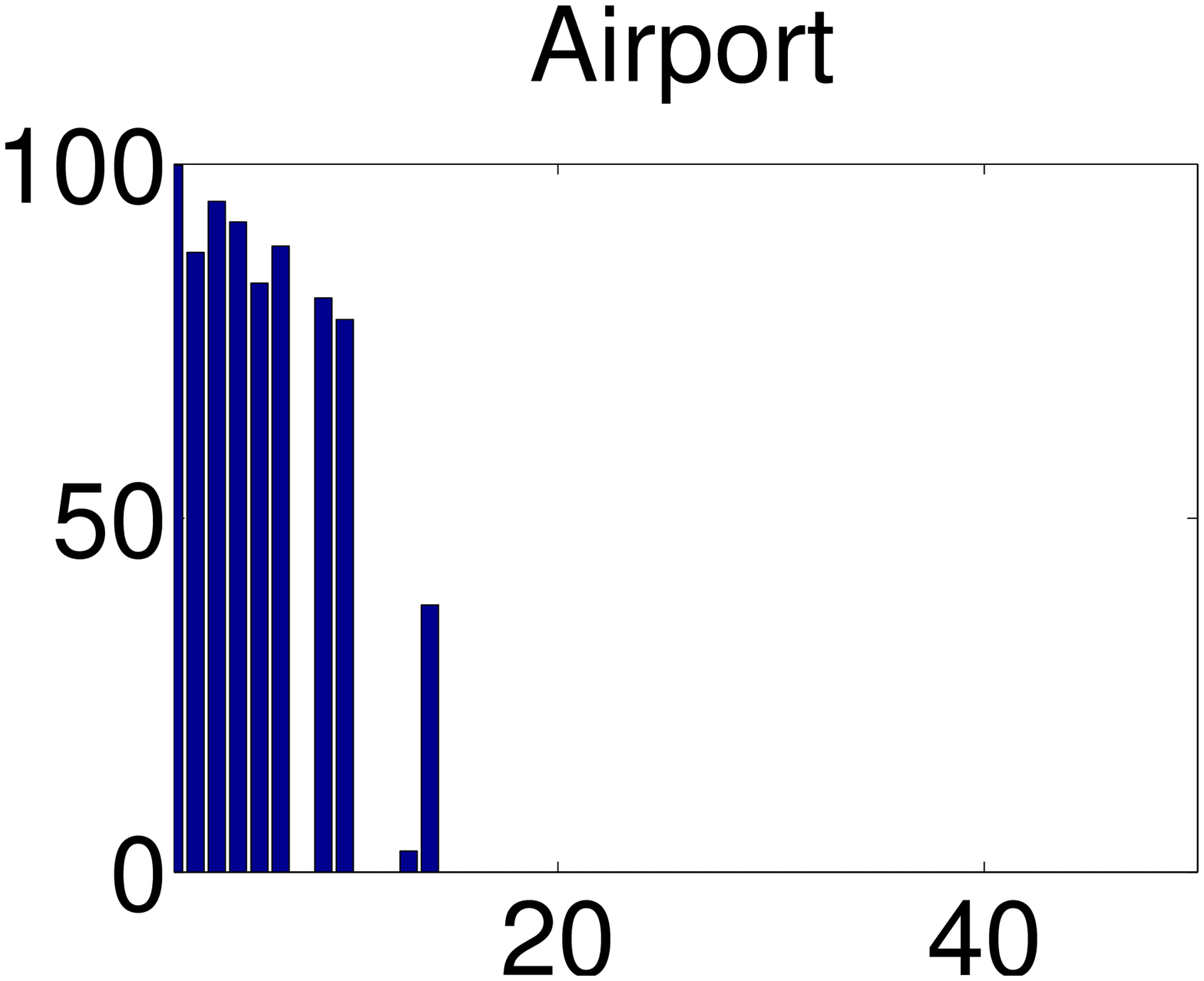} }

	\end{center}
	\caption{CORCONDIA for \cpnmu}
	\label{fig:real_fro}
\end{figure*}

\begin{figure*}[!ht]
	\begin{center}
		\subfigure[]{ \includegraphics[width = \subfigurewidth]{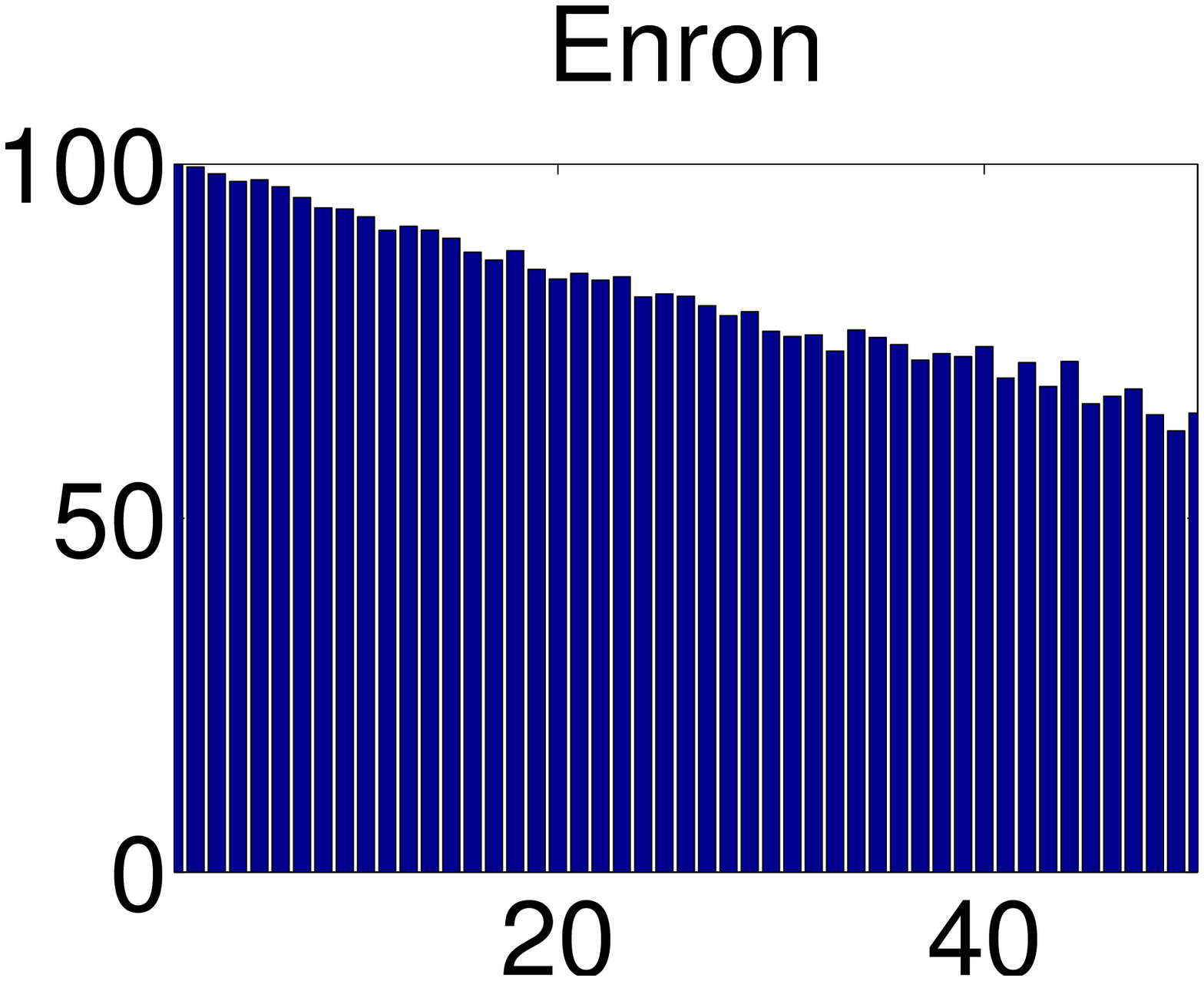} }
		\subfigure[]{ \includegraphics[width = \subfigurewidth]{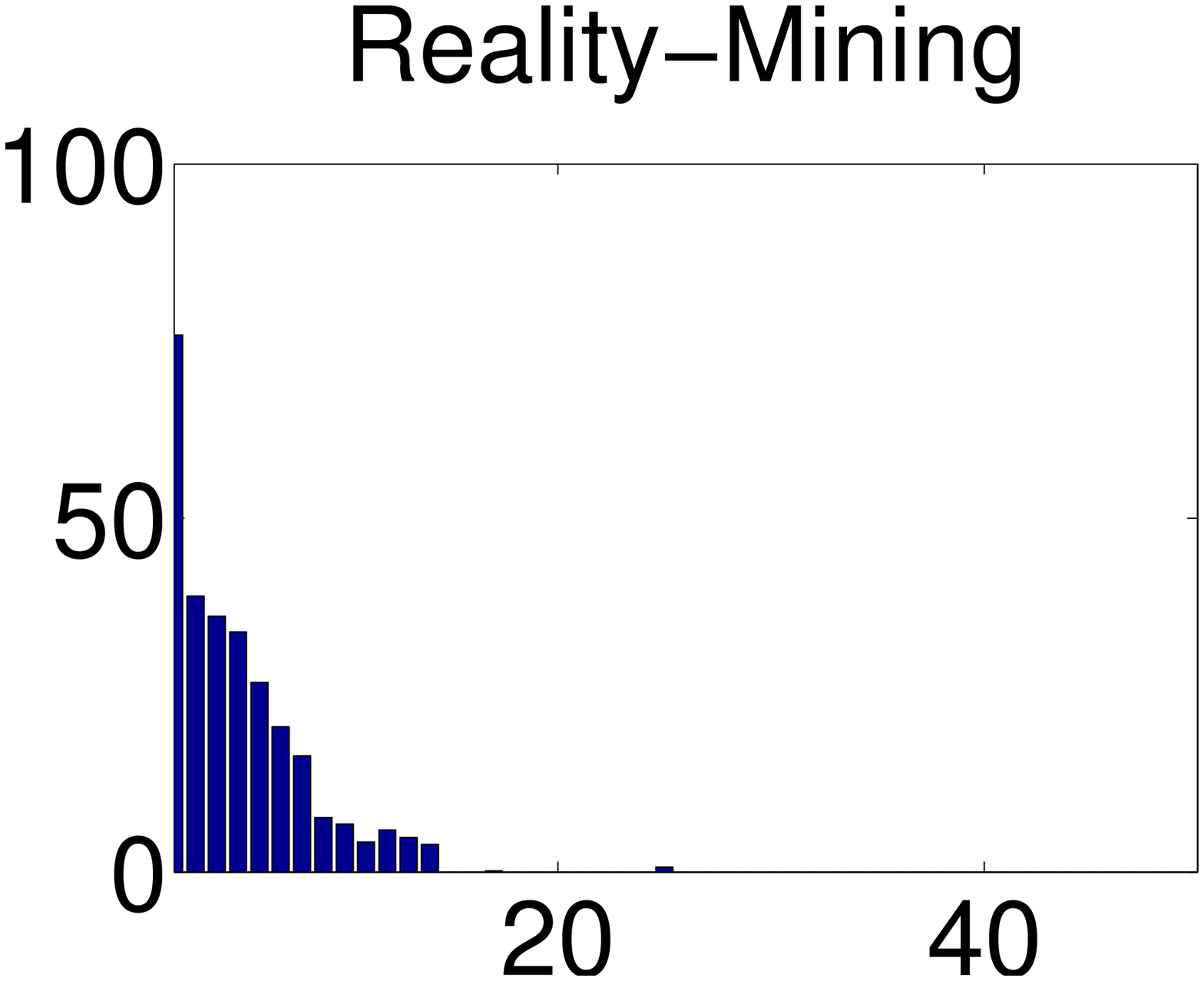} }
		\subfigure[]{ \includegraphics[width = \subfigurewidth]{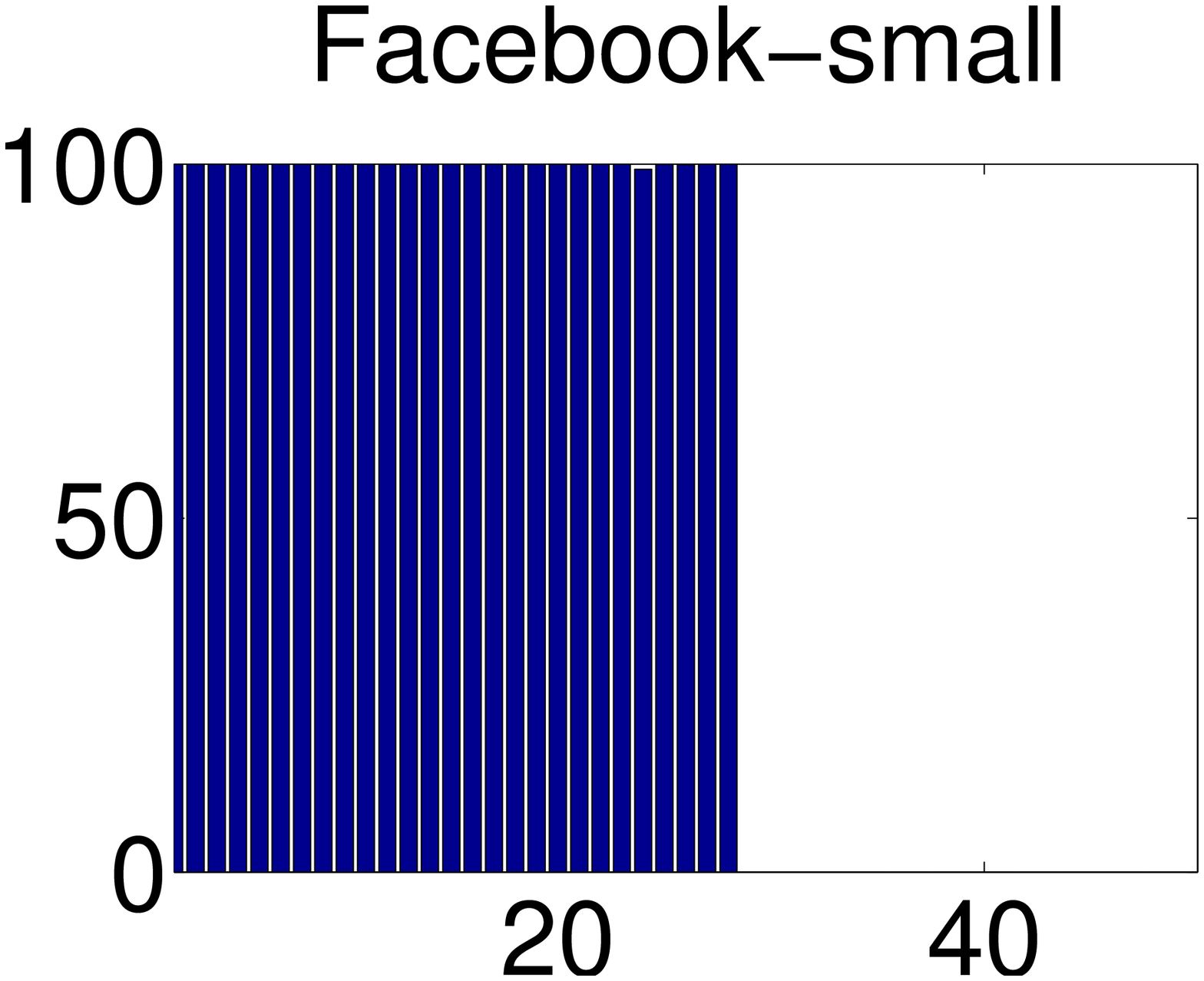} }
		\subfigure[]{ \includegraphics[width = \subfigurewidth]{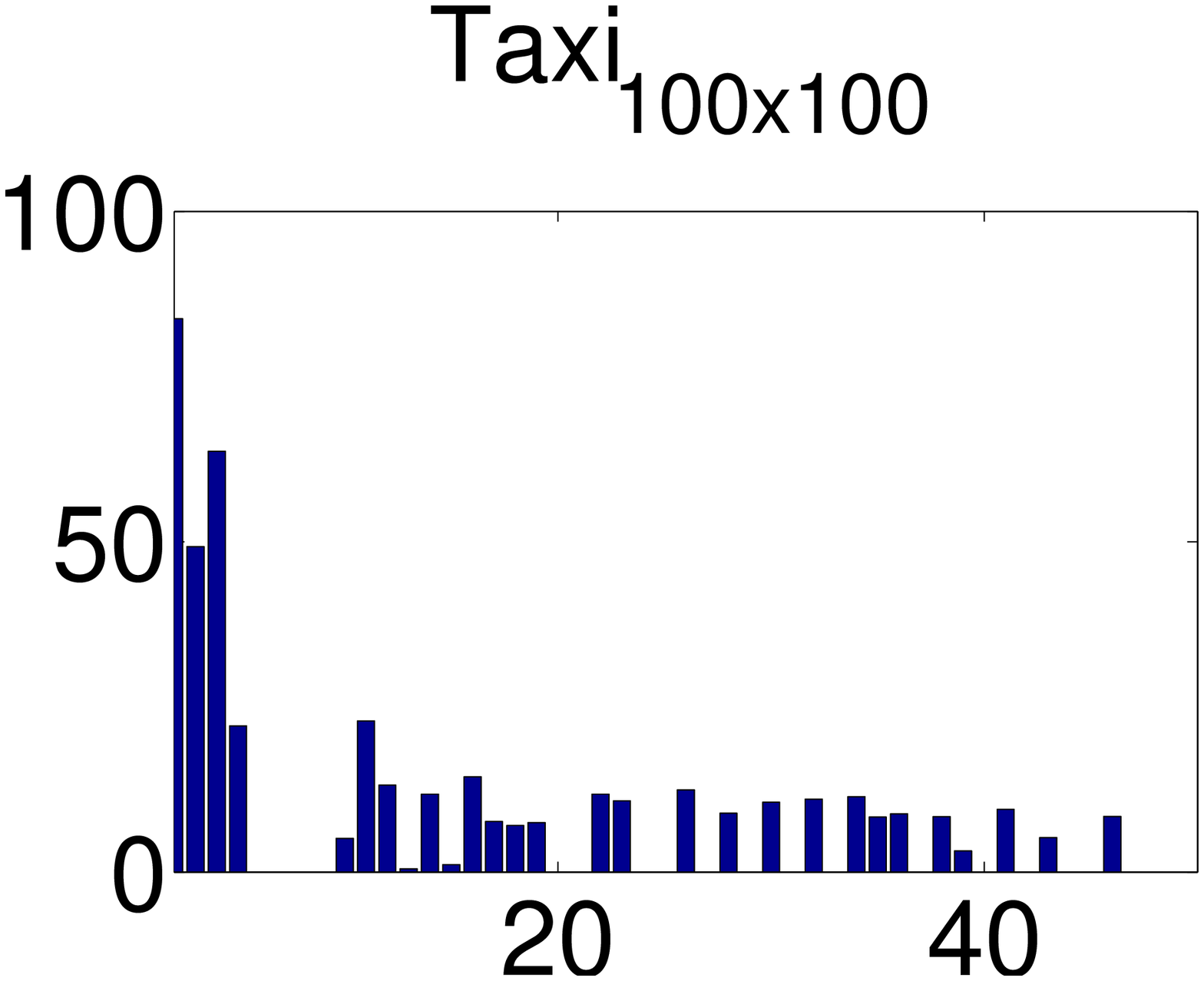} }
		\subfigure[]{ \includegraphics[width = \subfigurewidth]{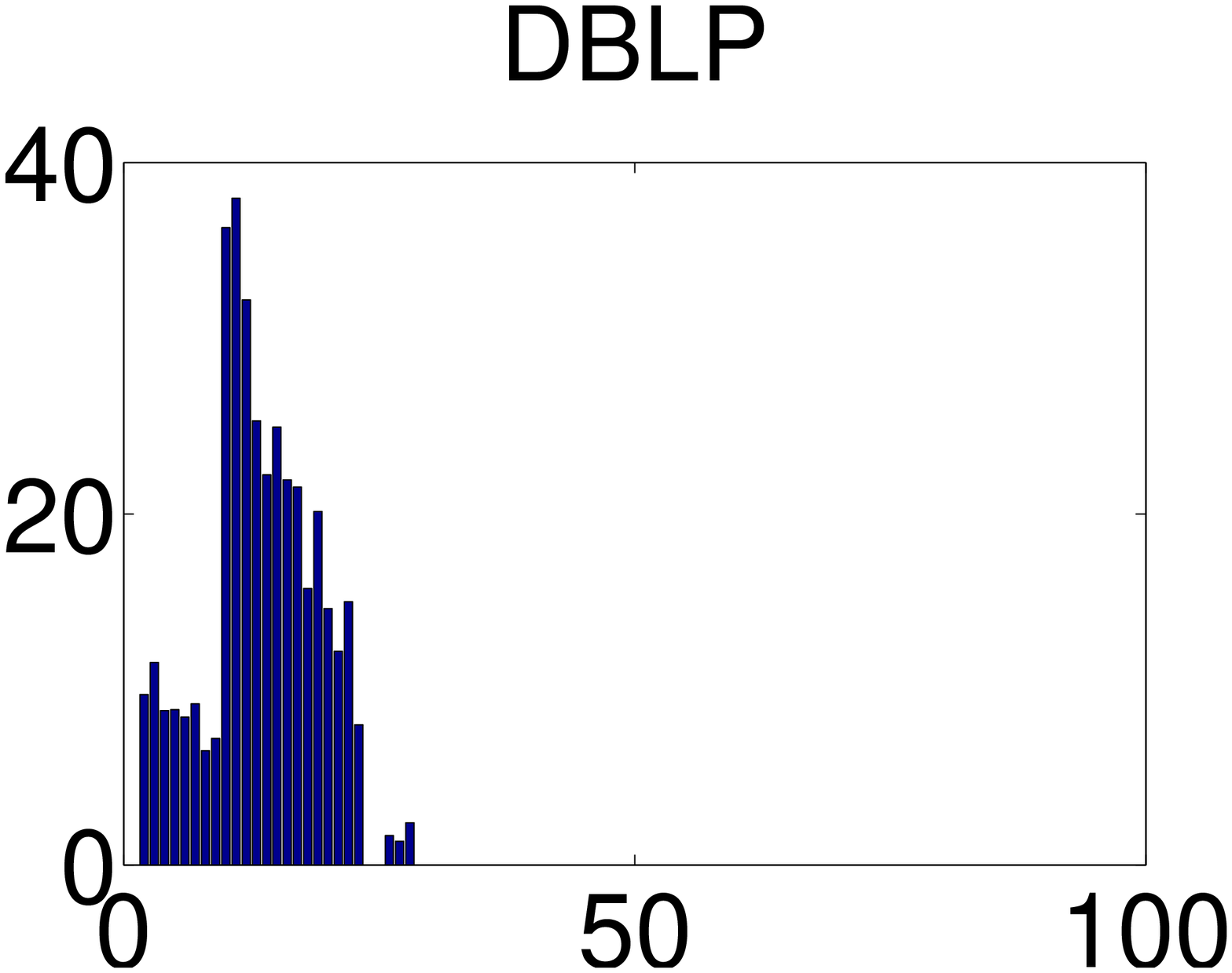} }
		\subfigure[]{ \includegraphics[width = \subfigurewidth]{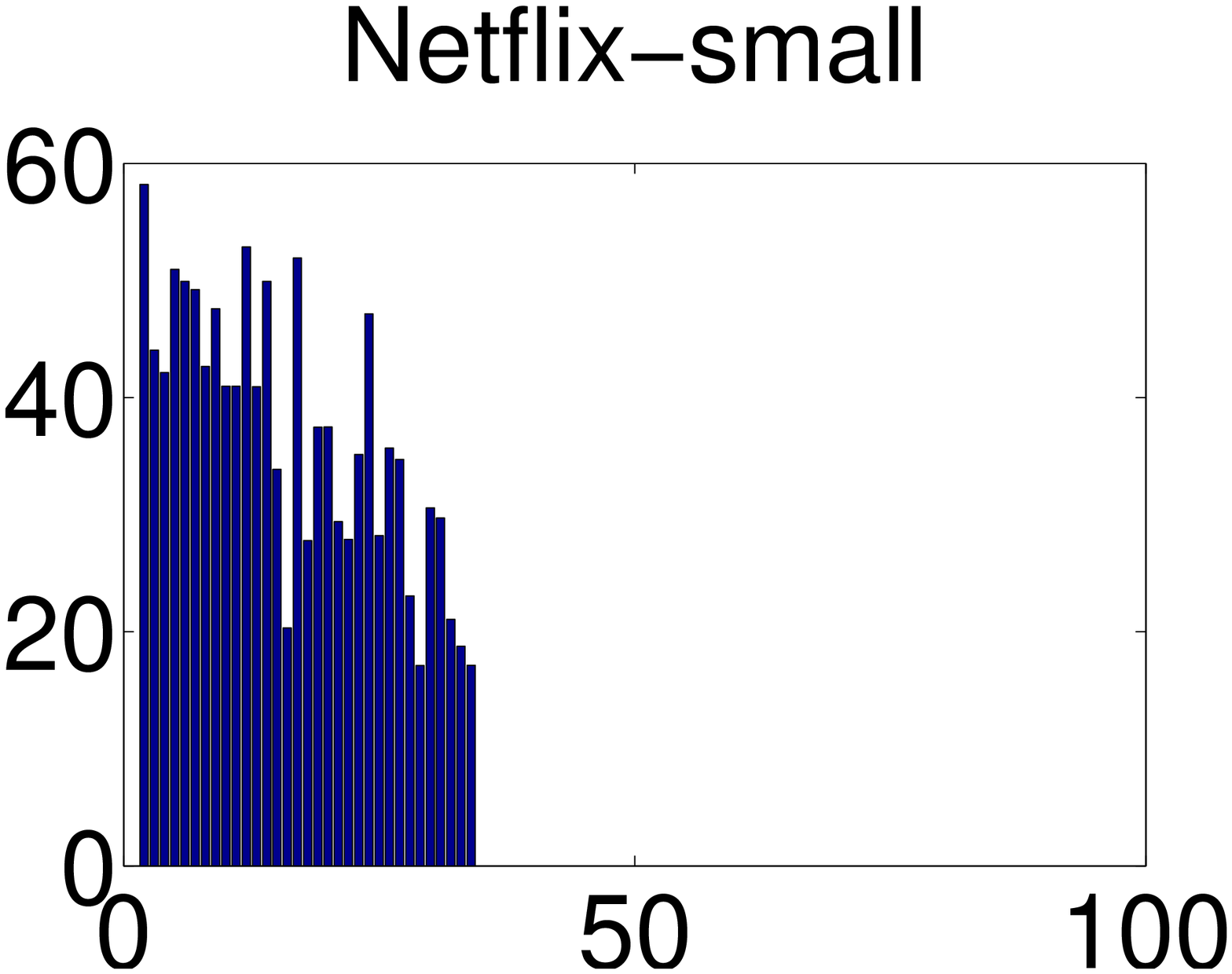} }
		\subfigure[]{ \includegraphics[width = \subfigurewidth]{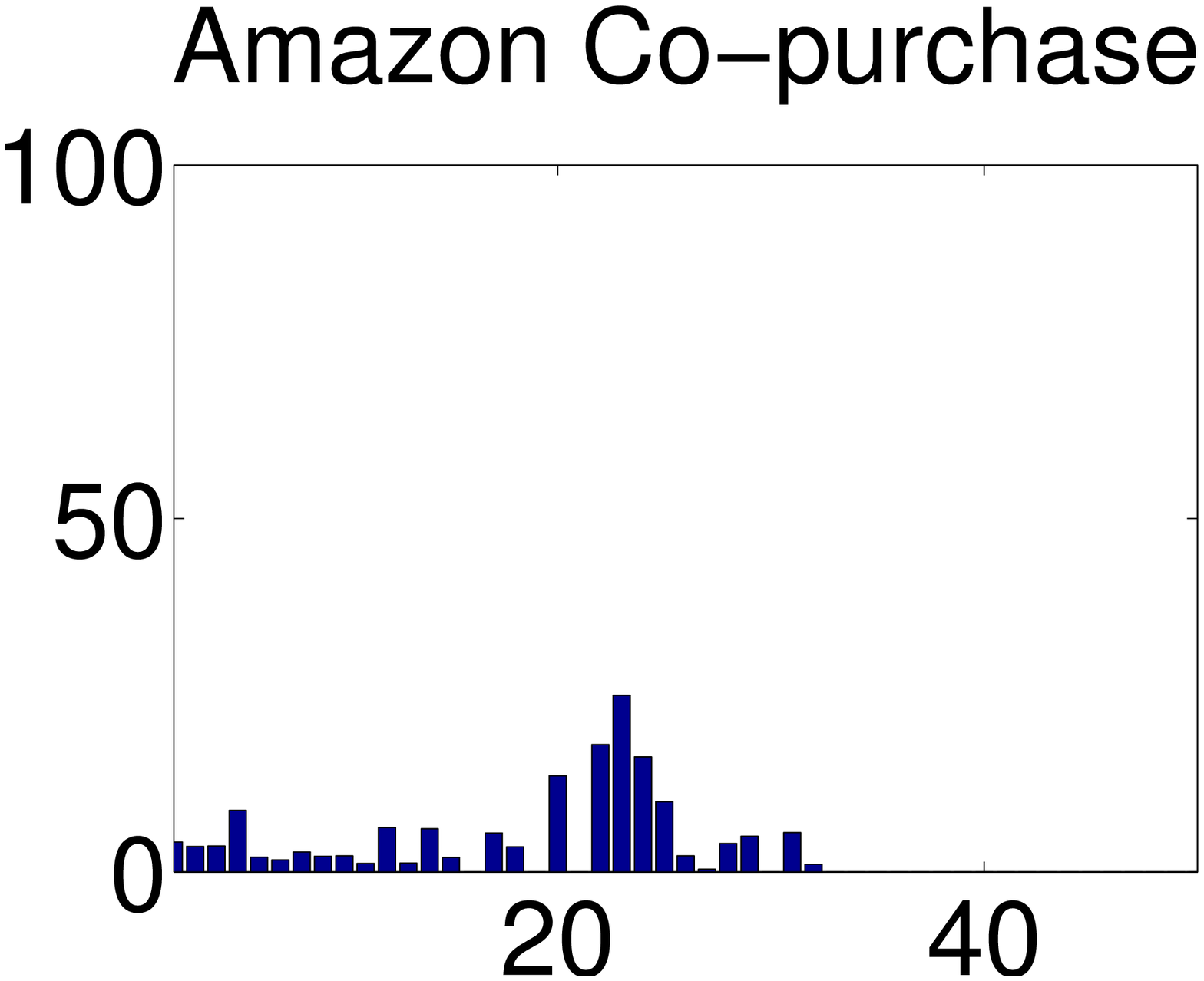} }
		\subfigure[]{ \includegraphics[width = \subfigurewidth]{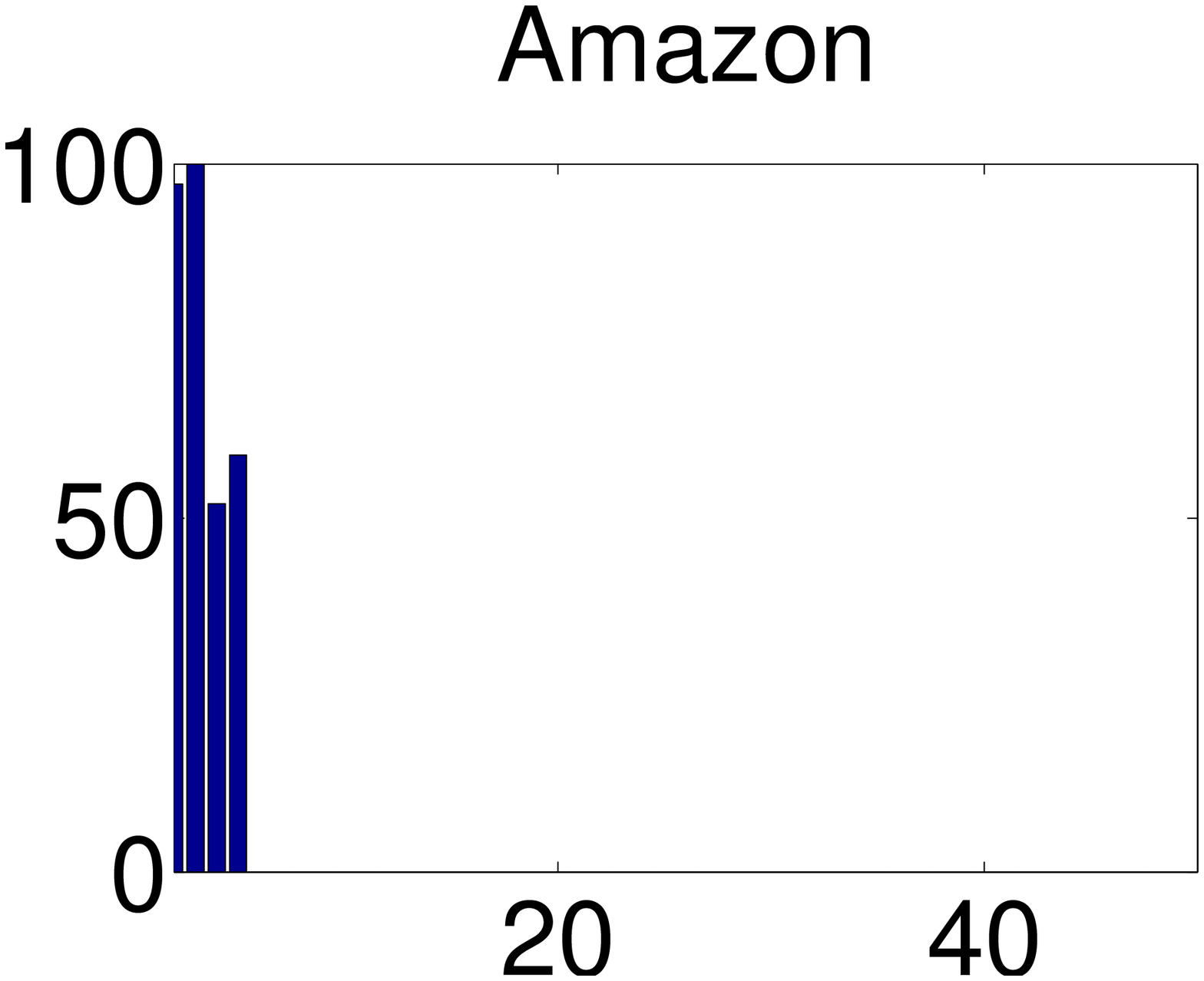} }
		\subfigure[]{ \includegraphics[width = \subfigurewidth]{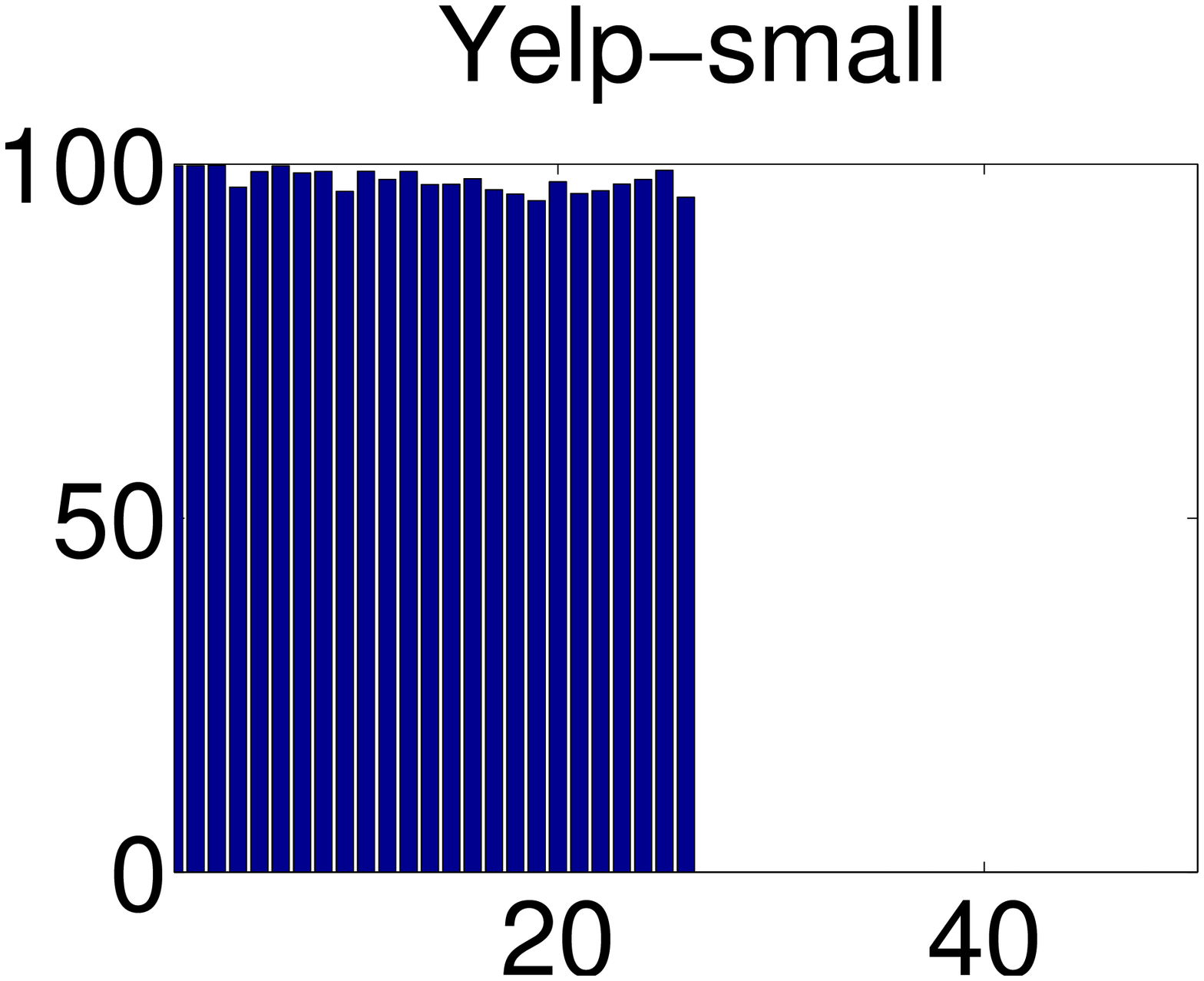} }
		\subfigure[]{ \includegraphics[width = \subfigurewidth]{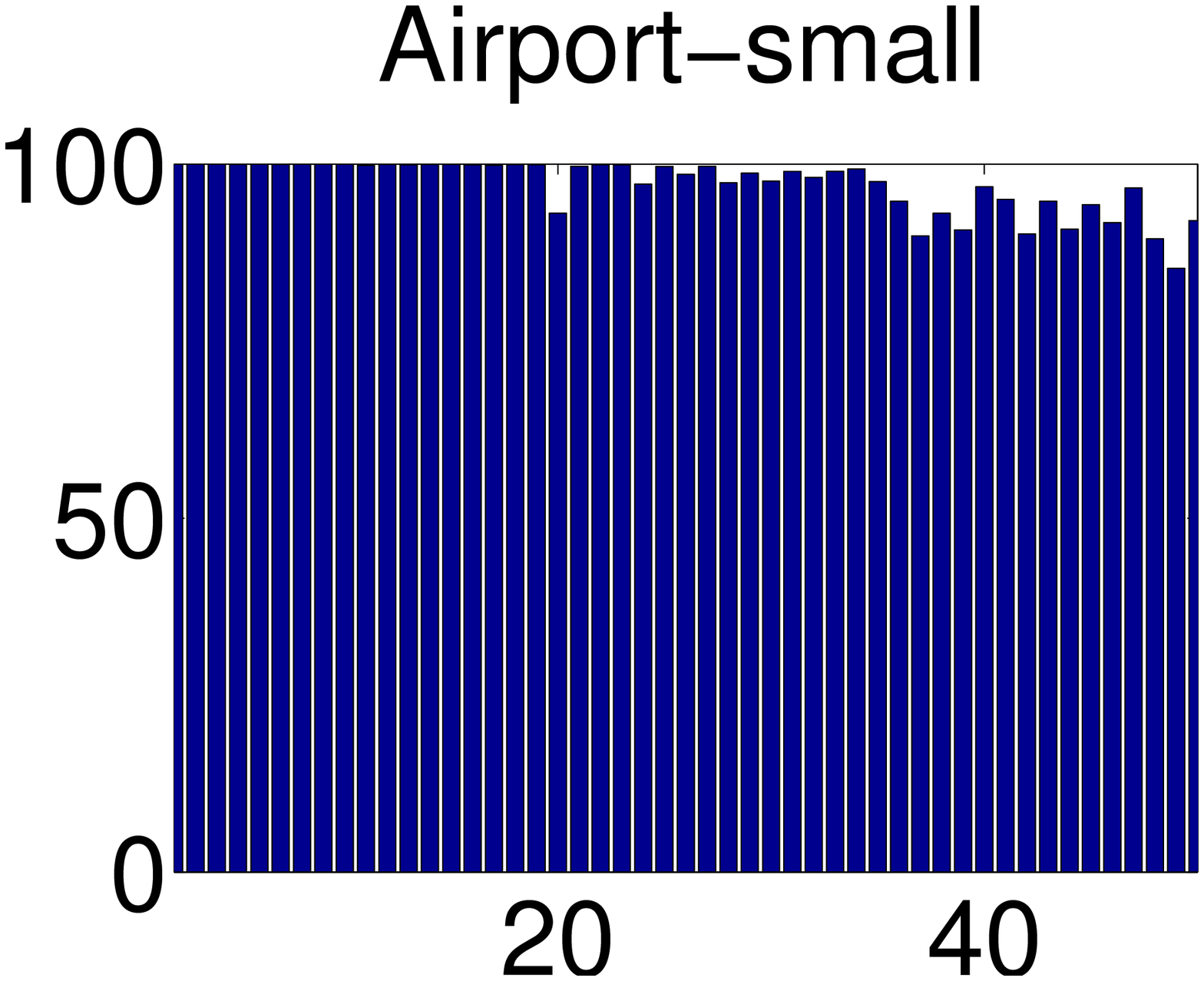} }

	\end{center}
	\caption{CORCONDIA for \cpapr}
	\label{fig:real_kl}
\end{figure*}

We ran our algorithms for $F=2\cdots50$, and truncated negative values to zero. For KL-Divergence  and datasets \facebook, \netflix, \yelp, and \airport we used smaller versions (first 500 rows for \netflix and \yelp, and first 1000 rows for \facebook and \airport), due to high memory requirements of Matlab; this means that the corresponding figures describe the rank structure of a smaller dataset, which might be different from the full one. Figure \ref{fig:real_fro} shows CORCONDIA when using Frobenius norm as a loss, and Fig. \ref{fig:real_kl} when using KL-Divergence. The way to interpret these figures is the following: assuming a \cpnmu (Fig. \ref{fig:real_fro}) or a \cpapr (Fig. \ref{fig:real_kl}) model, each figure shows the modelling quality of the data for a given rank. This sheds light to the rank structure of a particular dataset (although that is not to say that it provides a definitive answer about its true rank). For the given datasets, we observe a few interesting differences in structure: for instance, \enron and \taxi in Fig. \ref{fig:real_fro} seem to have good quality for a few components. On the other hand, \reality, \dblp, and \amazonmeta have reasonably acceptable quality for a larger range of components, with the quality decreasing as the number gets higher. Another interesting observation, confirming recent results in \cite{zhang2014understanding}, is that \yelp seems to be modelled better using a high number of components. Figures that are all-zero merely show that no good structure was detected for up to 50 components, however, this might indicate that such datasets (e.g. \netflix) have an even higher number of components.
Finally, contrasting Fig. \ref{fig:real_kl} to Fig. \ref{fig:real_fro}, we observe that in many cases using the KL-Divergence is able to discover better structure than the Frobenius norm (e.g. \enron and \amazonco).

\subsection{\methodplain in practice}
\label{sec:study2}
We used \method to analyze two of the datasets shown in Table \ref{tab:datasets}. In the following lines we show our results.

\subsubsection{Analyzing \taxi}
The data we have span an entire week worth of measurements, with temporal granularity of minutes. First, we tried quantizing the latitude and longitude into a $1000\times1000$ grid; however, \method warned us that the decomposition was not able to detect good and coherent structure in the data, perhaps due to the extremely sparse variable space of our grid. Subsequently, we modelled the data using a $100\times100$  grid and \method was able to detect good structure. In particular, \method output 8 rank-one components, choosing Frobenius norm as a loss function.

In Figure \ref{fig:beijing_taxi} we show 4 representative components of the decomposition. In each sub-figure, we overlay the map of Beijing with the coordinates that appear to have high activity in the particular component; every sub-figure also shows the temporal profile of the component. The first two components (Fig. \ref{fig:beijing_taxi}(a), (b)) spatially refer to a similar area, roughly corresponding to the tourist and business center in the central rings of the city. The difference is that Fig. \ref{fig:beijing_taxi}(a) shows high activity during the weekdays and declining activity over the weekend (indicated by five peaks of equal height, followed by two smaller peaks), whereas Fig. \ref{fig:beijing_taxi}(b) shows a slightly inverted temporal profile, where the activity peaks over the weekend; we conclude that Fig. \ref{fig:beijing_taxi}(a) most likely captures business traffic that peaks during the week, whereas Fig. \ref{fig:beijing_taxi}(b) captures tourist and leisure traffic that peaks over the weekend. The third component (Fig. \ref{fig:beijing_taxi}(c)) is highly active around the Olympic Center and Convention Center area, with peaking activity in the middle of the week.
Finally, the last component (Fig. \ref{fig:beijing_taxi}(d) ) shows high activity only outside of Beijing's international airport, where taxis gather to pick-up customers; on the temporal side, we see daily peaks of activity, with the collective activity dropping during the weekend, when there is significantly less traffic of people coming to the city for business.
By being able to analyze such trajectory data into highly interpretable results, we can help policymakers to better understand the traffic patterns of taxis in big cities, estimate high and low demand areas and times and optimize city planning in that respect. There has been very recent work \cite{Wang2014travel} towards the same direction, and we view our results as complementary.

\begin{figure*}[!ht]
	\begin{center}
		\subfigure[{\bf Tourist \& Business Center}: High activity during weekdays, low over the weekend]{\includegraphics[width = 0.495\textwidth]{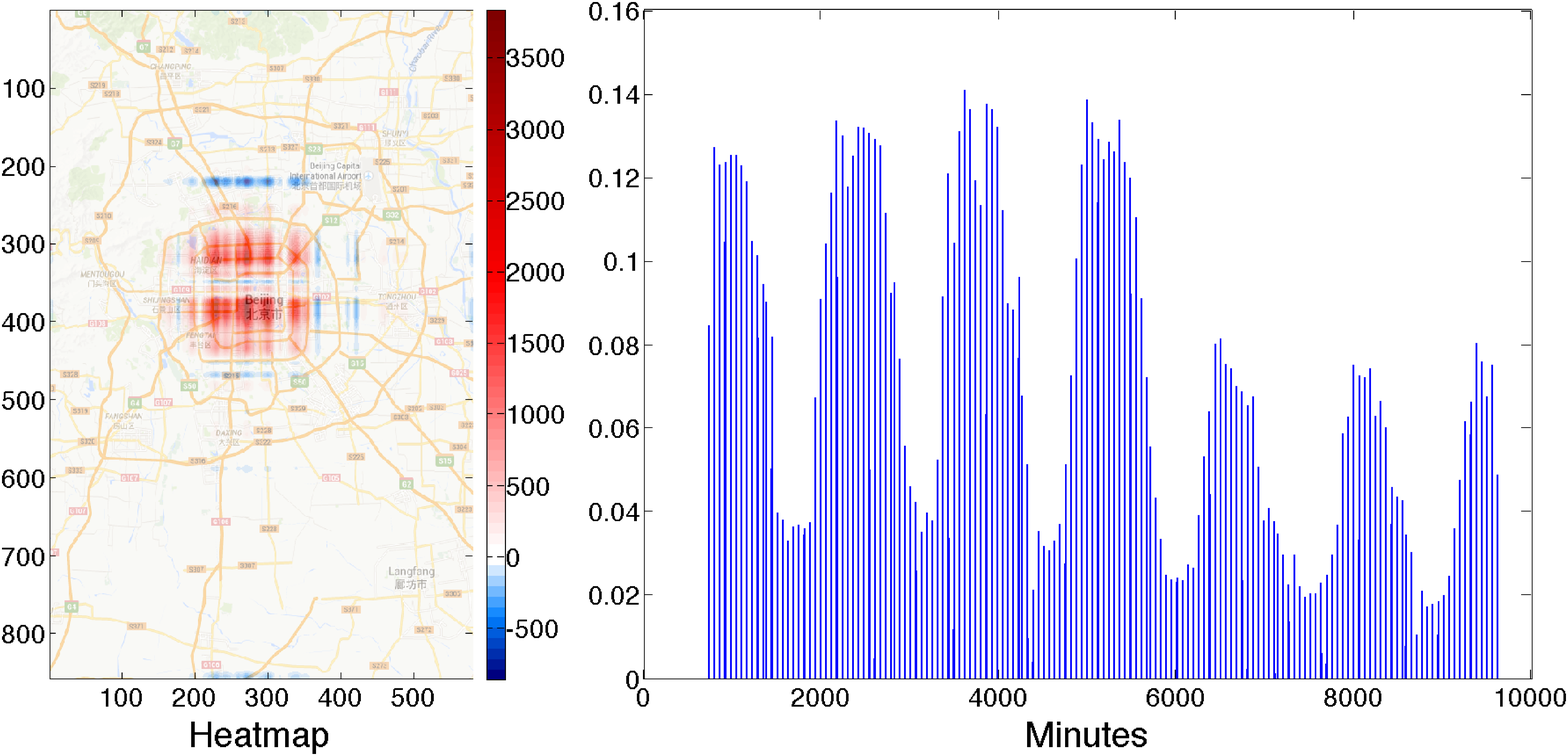}}
		\subfigure[{\bf Downtown}: Consistent activity over the week]{\includegraphics[width = 0.495\textwidth]{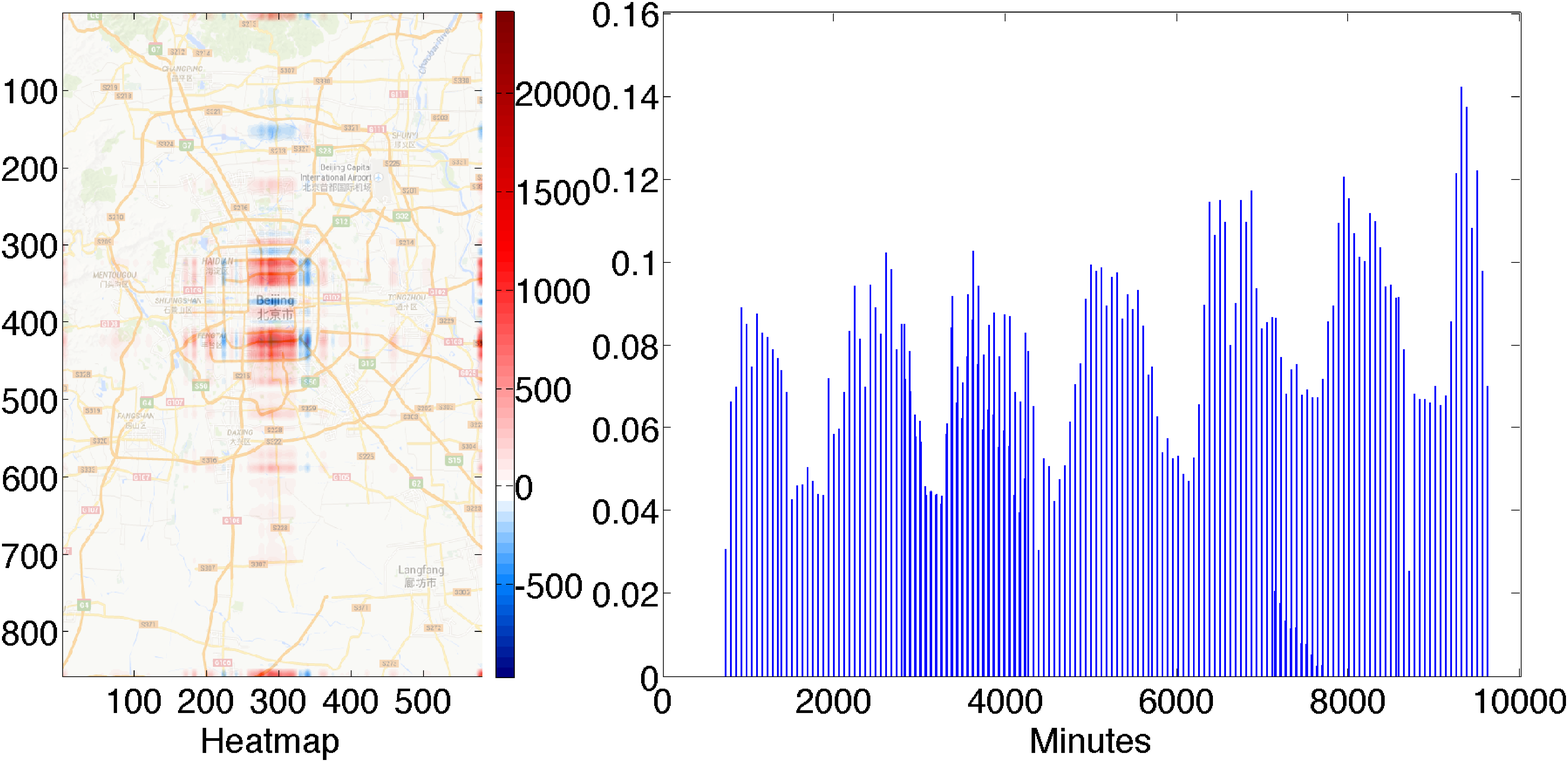}}
		\subfigure[{\bf Olympic Center}: Activity peak during the week]{\includegraphics[width = 0.495\textwidth]{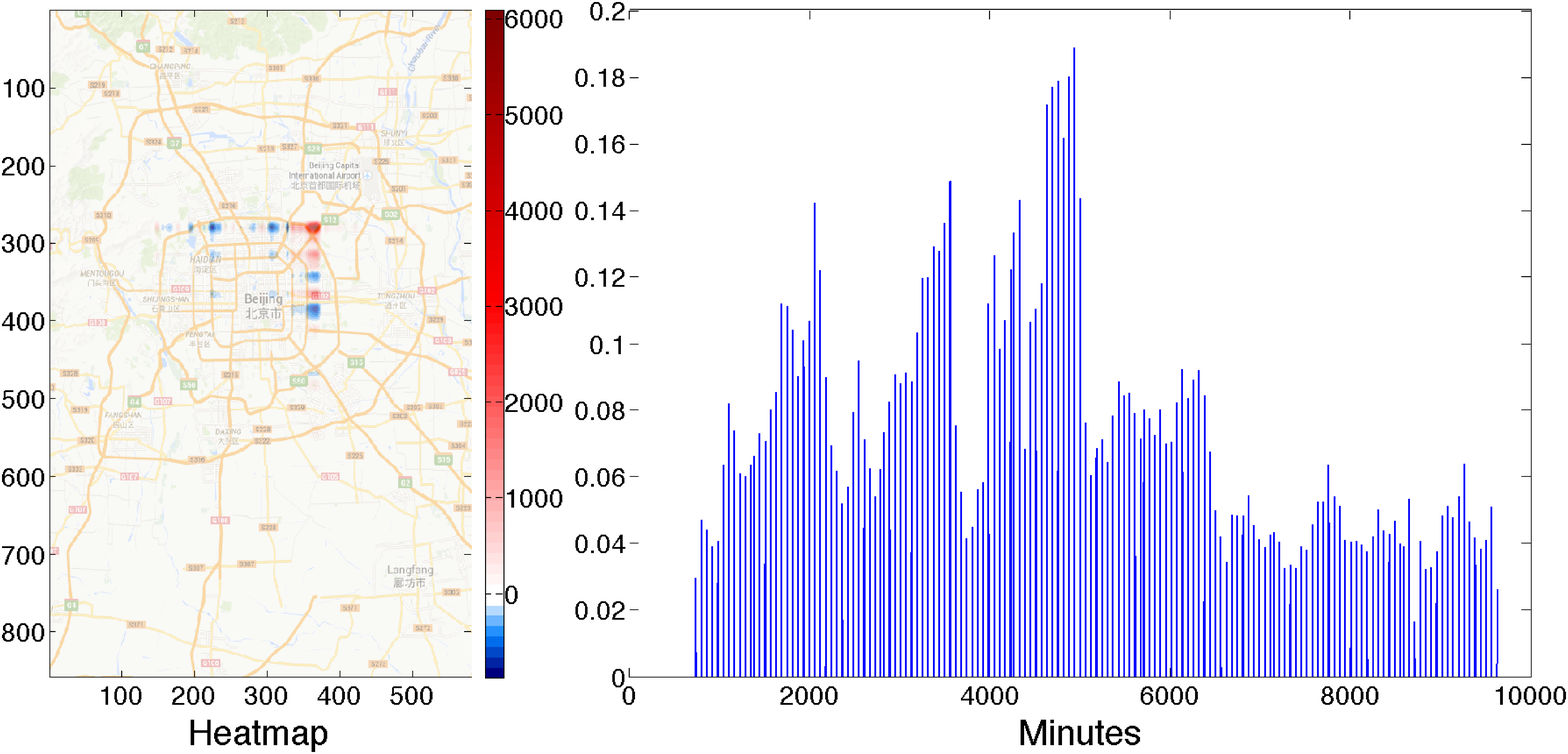}}
		\subfigure[{\bf Airport}: High activity during weekdays, low over the weekend]{\includegraphics[width = 0.495\textwidth]{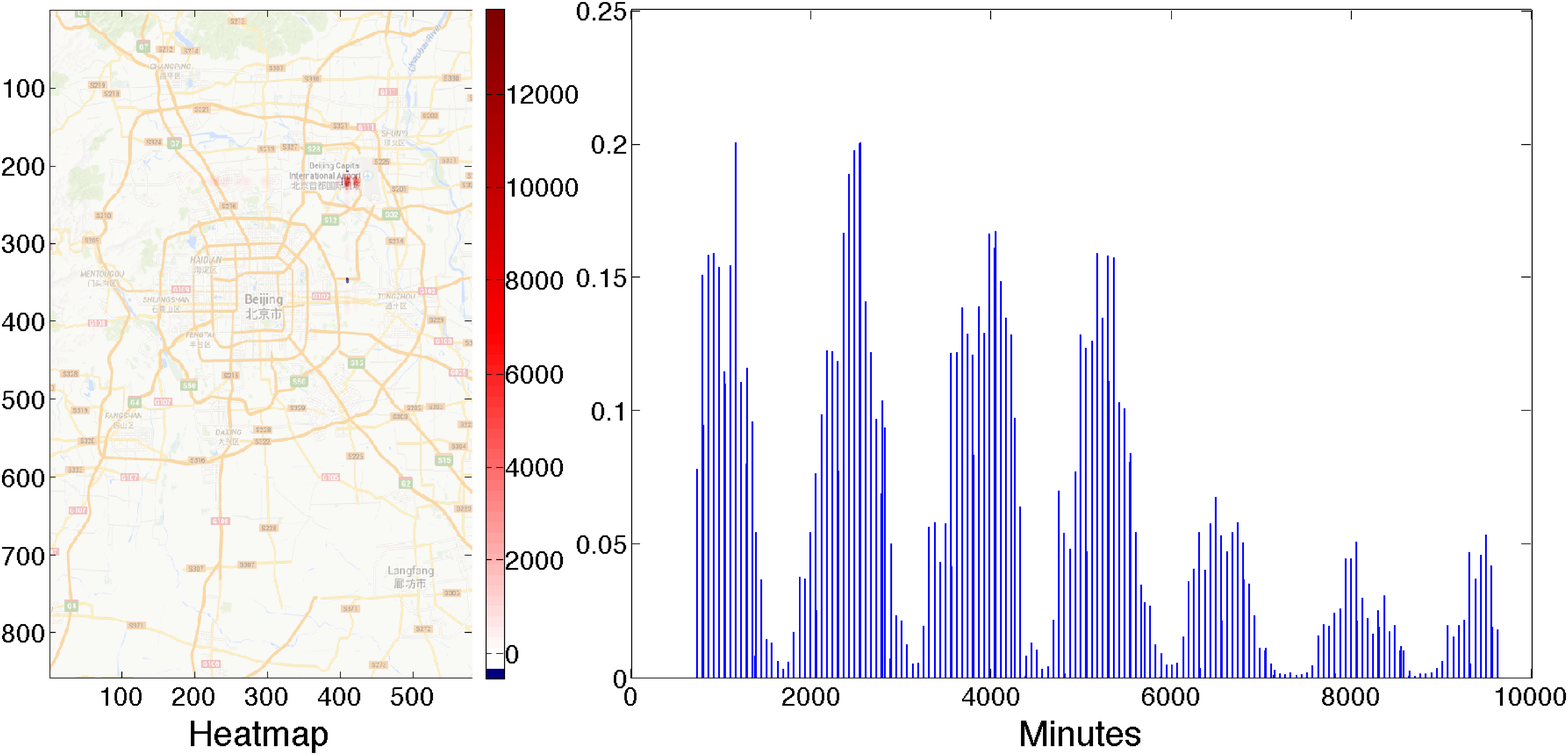}}
	\end{center}
	\caption{Latent components of the \taxi dataset, as extracted using \method.}
	\label{fig:beijing_taxi}
\end{figure*}

\subsubsection{Analyzing \amazonco}
\begin{table*}[!htf]
	\begin{center}
	\scriptsize
		\begin{tabular}{|c|c|c|}
		\hline
		{\bf Cluster type} & {\bf Products} & {\bf Product Types}  \\\hline
		\multirow{4}{*}{\#1 Self Improvement} & Resolving Conflicts At Work : A Complete Guide for Everyone on the Job  & Book  \\
			&	How to Kill a Monster (Goosebumps)   &	Book \\
			&	Mensa Visual Brainteasers & Book \\
			&	Learning in Overdrive: Designing Curriculum, Instruction, and Assessment from Standards : A Manual for Teachers &	 Book \\\hline
		\multirow{2}{*}{\#2 Psychology, Self Improvement} & Physicians of the Soul: The Psychologies of the World's Greatest Spiritual Leaders    & Book \\
			&	The Rise of the Creative Class: And How It's Transforming Work, Leisure, Community and Everyday Life & Book \\ \hline
		\multirow{3}{*}{\#3 Technical Books} & Beginning ASP.NET Databases using C\#  & Book \\
			&	BizPricer Business Valuation Manual w\/Software	& Book \\
			&	Desde Que Samba E Samba & Music \\ \hline
		\multirow{3}{*}{\#4 History} & War at Sea: A Naval History of World War II    & Book \\
			&	Jailed for Freedom: American Women Win the Vote		&	Book \\
			&	The Perfect Plan (7th Heaven)	&	Book \\ \hline
	\end{tabular}
	\end{center}
	\caption{Latent components of the \amazonco dataset, as extracted using \method}
	\label{tab:amazon}
\end{table*}
This dataset records pairs of products that were purchased together by the same customer on Amazon, as well as the category of the first product in the pair. This dataset, as shown in Figures \ref{fig:real_fro}(g) and \ref{fig:real_kl} does not have perfect trilinear structure, however a low rank trilinear approximation still offers reasonably good insights for product recommendation and market basket analysis.
By analyzing this dataset, we seek to find coherent groups of products that people tend to purchase together, aiming for better product recommendations and suggestions. For the purposes of this study, we extracted a small piece of the co-purchase network of 256 products. \method was able to extract 24 components by choosing KL-Divergence as a loss. On Table \ref{tab:amazon} we show a representative subset of our resulting components (which were remarkably sparse, due to the KL-Divergence fitting by \cpapr). We observe that products of similar genre and themes tend to naturally cluster together. For instance, cluster \#1 contains mostly self improvement books. We also observe a few topical outliers, such as the book \texttt{How to Kill a Monster (Goosebumps)} in cluster \#1, and CD \texttt{Desde Que Samba E Samba} in cluster \#3 that contains Technical / Software Development books.
\section{Related Work}
\label{sec:related}
\paragraph{Tensors and their data mining applications}
One of the first applications was on web mining, extending the popular HITS algorithm \cite{kolda2005higher}. There has been work on analyzing citation networks (such as DBLP) \cite{kolda2008scalable}, detecting anomalies in computer networks\cite{kolda2008scalable,maruhashi2011multiaspectforensics,papalexakis2012parcube}, extracting patterns from and completing Knowledge Bases \cite{kang2012gigatensor,papalexakis2012parcube,chang2014typed}  and analyzing time-evolving or multi-view social networks. \cite{bader2007temporal,kolda2008scalable,lin2009metafac,papalexakis2012parcube,araujo2014com2,jiang2014fema},
The long list of application continues, with extensions of Latent Semantic Analysis \cite{cai2006tensor,chang2013multi}, extensions of Subspace Clustering to higher orders \cite{huang2008simultaneous}, Crime Forecasting \cite{mu2011empirical}, Image Processing \cite{liu2013tensor}, mining Brain data \cite{davidson2013network,he2014dusk,papalexakisturbo}, trajectory and mobility data  \cite{zheng2010collaborative,Wang2014travel},  and bioinformatics \cite{ho2014marble}.
\paragraph{Choosing the right number of components}
As we've mentioned throughout the text, CORCONDIA \cite{bro1998multi} is using properties of the PARAFAC decomposition in order to hint towards the right number of components. In \cite{papalexakis2015fastcorcondia}, we introduce a scalable algorithm for CORCONDIA (under the Frobenius norm). Moving away from the PARAFAC decompostion, Kiers and Kinderen \cite{kiers2003fast} introduce a method for choosing the number of components for Tucker3.
There has been recent work using Minimum Description Length (MDL): In \cite{araujo2014com2} the authors use MDL in the context of community detection in time-evolving social network tensors, whereas in \cite{metzler2015clustering}, Metzler and Miettinen use MDL to score the quality of components for a binary tensor factorization. Finally, there have also been recent advances using Bayesian methods \cite{zhaobayesian} in order to automatically decide the number of components.

\section{Conclusions}
\label{sec:conclusions}
In this paper, we work towards an automatic, unsupervised tensor mining algorithm that minimizes user intervention. Our main contributions are:
\begin{itemize}\setlength{\itemsep}{0.0mm}
	\item {\bf Algorithms} We propose \method, a novel automatic and unsupervised tensor mining algorithm, which can provide quality characterization of the solution. We extend the highly intuitive heuristic of \cite{bro1998multi} for KL-Divergence loss, providing an efficient algorithm.
	\item{\bf Evaluation \& Discovery} We evaluate our methods in synthetic data, showing their superiority compared to the baselines, as well as a wide variety of real datasets. Finally, we apply \method to two real datasets discovering meaningful patterns (Section \ref{sec:study2}).
\end{itemize}

\section{Acknowledgements}
{\scriptsize
Research was supported by the National Science Foundation Grant No. IIS-1247489. Any opinions, findings, and conclusions or recommendations expressed in this material are those of the author(s) and do not necessarily reflect the views of the funding parties.
The author would like to thank Professors Christos Faloutsos, Nicholas Sidiropoulos, and Rasmus Bro for valuable conversations, and Miguel Araujo for sharing the \airport dataset.
}

\balance
\bibliographystyle{IEEEbib}
\bibliography{BIB/vagelis_refs.bib}

\end{document}